%% file: main.tex
\documentclass{article}
\usepackage{arxiv}
\usepackage[utf8]{inputenc} 
\usepackage[T1]{fontenc}    
\usepackage{hyperref}       
\usepackage{url}            
\usepackage{booktabs}       
\usepackage{amsfonts}       
\usepackage{nicefrac}       
\usepackage{microtype}      
\usepackage{lipsum}		
\usepackage{graphicx}
\usepackage{natbib}
\usepackage{doi}

\usepackage{amsmath}
\usepackage{amsthm}
\newtheorem{thm}{Theorem}[section]

\usepackage{pgfplots}
\pgfplotsset{compat=1.17}
\usepackage{xspace}
\usepackage{amsfonts}
\usepackage{algorithm}
\usepackage{algorithmic}
\usepackage{lscape} 

\usepackage{multirow}
\usepackage{float}
\usepackage{colortbl}
\newcommand{\ignore}[1]{}

\AtBeginDocument{%
  \providecommand\BibTeX{{%
    \normalfont B\kern-0.5em{\scshape i\kern-0.25em b}\kern-0.8em\TeX}}}

\begin{document}

\title{Evolutionary Multi-Objective Algorithms for the Knapsack Problems with Stochastic Profits}

\date{}

\author{
        {Kokila Perera} \\
        Optimisation and Logistics\\
	School of Computer and Mathematical Sciences\\
        The University of Adelaide\\
        Adelaide, Australia \\
        \And
	{Aneta Neumann} \\
 Optimisation and Logistics\\
	School of Computer and Mathematical Sciences\\
        The University of Adelaide\\
        Adelaide, Australia \\
        \And
	{Frank Neumann} \\
 Optimisation and Logistics\\
	School of Computer and Mathematical Sciences\\
        The University of Adelaide\\
        Adelaide, Australia \\
}

\maketitle
\begin{abstract}

Evolutionary multi-objective algorithms have been widely shown to be successful when utilized for a variety of stochastic combinatorial optimization problems. Chance constrained optimization plays an important role in complex real-world scenarios, as it allows decision makers to take into account the uncertainty of the environment. We consider a version of the knapsack problem with stochastic profits to guarantee a certain level of confidence in the profit of the solutions. We introduce the multi-objective formulations of the profit chance constrained knapsack problem and design three bi-objective fitness evaluation methods that work independently of the specific confidence level required. We evaluate our approaches using well-known multi-objective evolutionary algorithms GSEMO and NSGA-II. In addition, we introduce a filtering method for GSEMO that improves the quality of the final population by periodically removing certain solutions from the interim populations based on their confidence level. We show the effectiveness of our approaches on several benchmarks for both settings where the knapsack items have fixed uniform uncertainties and uncertainties that are positively correlated with the expected profit of an item.
\end{abstract}


\section{Introduction}
\label{section:introduction}
Real world optimization problems involve often uncertain properties imposed by some stochastic components in the problem or noise in its environment \cite{peng:tel-02303045, He_Shao_ICIECS_2009}. Such uncertainties must be considered in the optimization to find reliable and useful solutions to problems. Chance constraints are a natural way to model uncertainties in problems. They can capture the effects of uncertain variables on the inequality constraints and formulate the optimization problems under uncertainties \cite{peng:tel-02303045, Neumann2020,DBLP:conf/aaai/DoerrD0NS20}. 
A chance constraint is usually an inequality constraint on some stochastic variable of the problem, which can be violated by a slight \emph{chance} (a small probability) when optimizing the problem \cite{neumann_sutton_foga_19,DBLP:conf/ppsn/NeumannN20}. Chance constraints enable one to set the desired confidence level of getting the maximal profit when implementing a solution, unlike using a solution from a deterministic approach. The study of problems with chance constraints leads to developing efficient applications to solve complex optimization problems and is crucial in developing real world solutions for complex problems, such as in the mining industry \cite{yue_stockpile_2021, mining_2021}, power systems \cite{GENG2019341}, communication systems \cite{Abe_2020_satcom} and transportation \cite{Kepaptsoglou2015_containership}. 

This work considers a variation of the classical knapsack problem (KP), where the items have stochastic profits and deterministic weights. As the weights of items are deterministic, the constraint on the weights remains the same as in the classical KP. An additional constraint is introduced to this problem to capture the uncertainty of the profit variable, which is a chance constraint on the profit \cite{Aneta_ppsn_22}. This constraint guarantees that the solution has maximal profit P and only drops below P for a small probability ($\alpha$). In summary, the optimization process for this problem is to find the solution(s) that maximize the profit P, subjected to the weight constraint and profit chance constraint, which guarantees that the profit P is obtained with probability at least $\alpha$.

In the literature, the KP with a chance constraint on the weight (as elements have deterministic profits and stochastic weights) is more prevalent \cite{yue_gecco_19,yue_gecco_20}. On the contrary, the profit chance constrained KP is explored only in one study, which appears in recent literature \cite{Aneta_ppsn_22}. This work considers single-objective evolutionary approaches to address this problem \cite{Aneta_ppsn_22}. This scenario makes it possible to identify the risk of not achieving the expected profit by implementing a particular solution for the problem and making better decisions in planning. This problem reflects a beneficial and valid real-world problem scenario like mine planning as mentioned in \cite{Aneta_ppsn_22}.

Evolutionary algorithms (EAs) perform well in addressing stochastic optimization problems \cite{Doer_Neumann_survey_21, DBLP:conf/ppsn/SinghB22}. Also, they can produce optimal or feasible solutions in a reasonable amount of time for complex combinatorial optimization problems like chance constrained knapsack and job shop scheduling problems. EAs may consider a different number of objectives depending on how the definition of the underlying problem. Usually, when an EA considers a single objective, it generates a single solution that optimizes the objective value. In contrast, a multi-objective EA generates a set of solutions that gives a trade-off between given objectives. Such a solution set makes it possible to have more insights into improving the algorithms and search space than having a single solution as the outcome \cite{coello2013evolutionary, kdeb01}. Therefore, multi-objective algorithms help one to make informed decisions on selecting a solution to implement. 

In this work, we explore the use of the multi-objective evolutionary approaches for the chance constrained KP with stochastic profits introduced in \cite{Aneta_ppsn_22}. Here we introduce multi-objective fitness evaluation for EAs to address this problem and methods to enhance their performance.

\subsection{Related Work}
\label{subsec:literature}
The use of evolutionary computation for chance constrained problems appears in the early literature \cite{He_Shao_ICIECS_2009,Loughlin_Ranjithan_gecco99, Liu_EA_for_CC, Masutomi_2013_jaciii}. Those works consider computationally expensive methods like simulations and sampling to cater for chance constraints. More recent studies have looked into tail-bound inequalities, which more efficiently deal with chance constraints \cite{Aneta_ppsn_22, yue_gecco_19, Hirad_ecai_20}. 

The chance constrained KP where the profits are deterministic and weights are stochastic is considered in several papers \cite{yue_gecco_19, yue_gecco_20}. \cite{yue_gecco_19} presents how to use well-known deviation inequalities: Chebyshev's inequality and Hoeffding bound to estimate the probability of constraint violation. In \cite{yue_gecco_20}, where the same KP variation is considered, they introduce problem-specific operators for EAs with both single- and multi-objective formulations. In the study \cite{Hirad_ecai_20}, dynamic chance constrained KP with stochastic profits is studied. In addition to the objective function on the profit of a given stochastic solution, a second objective is introduced to address the dynamic capacity constraint. It captures the minimal capacity bound for the solution that meets the chance constraints.

Run-time analysis is an essential topic in studying problems with chance constraints. The first paper on run time analysis for chance constraint problems considers the KP with stochastic weights \cite{neumann_sutton_foga_19}. This work considers different cases of that problem and studies the run time of (1+1) EA for them. In \cite{yue_runtime_gecco_21}, they perform the run time analysis of simple EAs for chance constrained KPs with uniform weights. The papers \cite{Neumann_ijcai2022} and \cite{CCMakespan_ppsn22} study the run time of simple EAs for different chance constrained problems. In\cite{Neumann_ijcai2022}, the authors consider single- and multi-objective EAs for chance constrained problems with normally distributed random problem variables. They also show how to use the proposed evolutionary approaches for chance constrained minimum spanning tree problems \cite{Neumann_ijcai2022}. In \cite{CCMakespan_ppsn22}, they analyze the run time of random local search and (1+1) EA for the chance constrained makespan problem. 

In the study \cite{Aneta_ppsn_22}, the authors simple EAs for the profit chance constrained KP. Those algorithms include (1+1) EA with standard bit-flip mutation and heavy-tail mutation operators and population based ($\mu$+1) EA, which uses a specific crossover operator for the KP. This study evaluates the performance of all these algorithms using the single objective fitness evaluation. The overall results show that (1+1) EA with heavy tail mutation operator significantly improved over other algorithms.

This work is motivated by the recent study on evolutionary multi-objective algorithms that compute trade-offs with respect to the expected value and the variance of solutions presented in \cite{Neumann_ijcai2022}. Here we explore the multi-objective formulations for the profit chance constrained KP based on that work. In addition to the variance of the solutions, we consider the standard deviation and also the count of elements in the solutions (under certain conditions) to introduce fitness evaluations for the problem. The significance of these fitness functions is that they can evaluate the fitness of a solution independent of any specific confidence level of its profit (i.e. independent of specific value for $\alpha$). Since this generates a set of solutions that gives a trade-off of specific objectives for each fitness formulation with each algorithm, it allows one to make more informed decisions when selecting a solution to implement. For example, to identify the solution that gives the best profit with a particular $\alpha$ value, we can calculate the profit of all the solutions for that value and select the solution that gives the best profit among the final population.

This study considers two well-known multi-objective EAs: GSEMO \cite{Giel2003_gsemo} and NSGA-II \cite{DebAPM02_nsga2}. Furthermore, we introduce a filtering method for GSEMO which aims to improve the quality of the final population resulting from the algorithm. This filtering method is applied regularly after a fixed number of evaluations in the evolutionary process. It considers whether a solution can be the best solution for any $\alpha \in [0.0, 1/2]$ considering the interim population at the time and, otherwise, removes it from the population. For all experimental settings, in addition to the two algorithms, GSEMO and NSGA-II, we consider this filtering method combined with GSEMO as a separate algorithm.

The structure of this paper is as follows. Section~\ref{sec:problem} introduces the problem formulation, and Section \ref{sec:problem} discusses the algorithms we consider in this study. Section \ref{sec:mo-formulation} discusses the multi-objective formulation, including the objective function on profit, fitness functions and how to use the probability bounds to estimate the confidence level of the solutions. Section \ref{sec:algorithms} and \ref{subsec:filtering}  discuss the EAs considered in this paper and the new Filtering method introduced in this work. Section \ref{sec:experiments} presents the details of the experimental settings and the results. Finally, Section \ref{sec:conclusion} gives the conclusions of this study.

\section{Problem Definition}
\label{sec:problem}
In the classical KP, the profit and weight values of the items are deterministic. Let the KP has n items $\{x_1, \ldots, x_i, \ldots, x_n\}$ with profit $p_i$ and weight $w_i$ and weight bound B. A solution to the problem can be represented as $\{0,1\}^n$ such that $x_i=1$ only when $x_i$ is selected in the solution. Given the profit of solution $x$ as $p(x) = \sum_{i=1}^n{p_i.x_i}$ and weight $w(x) = \sum_{i=1}^n{w_i.x_i}$, the classical KP is to find the solution $x^*$ that maximize $p(x^*)$ subjected to the weight constraint $w(x^{*}) \leq B$.

In this work, we consider a stochastic version of the classical KP where the profits of the knapsack items are stochastic while weights remain deterministic. 
Therefore the profit of a solution will be uncertain and may vary from the expected maximal profit value. A chance constraint on the profit is used to capture the stochastic behaviour of the problem. This constraint ensures that for each feasible solution $x$, the maximum probability that the profit will drop below profit value P is only a small probability $0 < \alpha <  1/2$. 

We can formally present this problem as follows:
\begin{eqnarray}
\label{eq:prob_randP_maxP}
    & \max P  \\
    \label{eq:prob_randP_alphaP}
    \mbox {s. t.}   & Pr(p(x) < P) \leq \alpha\\ 
    \label{eq:prob_randP_weightB}
    and  & w(x) \leq B
\end{eqnarray}

where $\alpha$ is a parameter determining the allowed failure probability which is a small value $<=1/2$. Equation \ref{eq:prob_randP_alphaP} specifies the chance constraint on the profit.

\subsection{Estimating Profit of a Solution}
\label{subsec:profit_estimates}
Computing the probability of violating a given chance constraint is intractable in general \cite{Doer_Neumann_survey_21}. Therefore, it is not straightforward to estimate the profit of a solution under uncertainties. However, tail bounds can be used to upper bound the probability of chance constraint violation\cite{Doer_Neumann_survey_21}.  In \cite{Aneta_ppsn_22}, the authors present  profit estimates for the problem based on the tail bounds: Chebyshev's inequality and Hoeffding bound. Each of these applies to the problem under certain conditions.

If the expectation and variance of the solutions' profits are known for a given problem instance, we can use Chebyshev's inequality to calculate the profit of a solution for it \cite{Doerr2020}. \cite{Aneta_ppsn_22} present a profit estimate in Equation \ref{eq:profit_cheby}, that ensure the profit chance constraint $Pr(p(x)<P)\leq \alpha$ is held. We can guarantee that the solution $x$ will give the profit $\hat{p}_{Cheb}(x,\alpha)$ except for a small probability $\alpha$ as follows:
\begin{eqnarray}
    \label{eq:profit_cheby}
    \hat{p}_{Cheb}(x, \alpha) = \mu(x) - \sqrt{(1-\alpha)/ \alpha} \cdot \sqrt{v(x)}.
\end{eqnarray}

The above equation gives a very general setting for the profit of a solution that can be used in any scenario where the expected value and the variance are known. For example, we can consider that the profits $p_i$ take a uniform distribution such that $p_i \in \{\mu_i-\delta,\mu_i+\delta\}$ which gives the expected value $\mu_i$ and variance $\frac{\delta^2}{3}.|x|_1$. 

The Hoeffding bound can be used to formulate a profit estimate if the profits are taken randomly from a uniform random distribution independent of each other \cite{Doerr2020}. From \cite{Aneta_ppsn_22}, we get the formulation for the profit of solution $x$ using the Hoeffding bound ($\hat{p}_{Hoef}(x,\alpha)$) as follows:
\begin{eqnarray}
    \label{eq:profit_cher}
    \hat{p}_{Hoef}(x,\alpha) = \mu(x) - \delta \cdot \sqrt{\ln(1/\alpha) \cdot 2 \cdot|x|_1}.
\end{eqnarray}

\section{Multi-Objective Formulation}
\label{sec:mo-formulation}
In this section, we introduce the multi-objective formulations of the profit chance constrained KP. As presented in Section \ref{sec:problem}, the optimal solution for the KP maximizes the profit subjected to the weight bound constraint. For the profit chance constrained KP with deterministic weights, the multi-objective formulation needs to consider the uncertainty of the profits of knapsack items. In the following subsections, we present the multi-objective formulations and the functions to estimate profits.

\subsection{Fitness Functions}
\label{subsec:fitness-functions}

In general, we need to ensure that the expectation of profit is maximized while the variance is minimized to maximize the profit of the solution. Considering this, we introduce the fitness function $g(x) = (\mu(x), v(x))$ that will produce a set of solutions that gives a trade-off between the two objectives, irrespective of $\alpha$ the guarantee of the profit of the individual solution's profit value. The formula for the two objectives is given in Equation \ref{eq:fitness-1-obj1} and \ref{eq:fitness-1-obj2} where $v_{max} = \sum_{i=1}^n {v_i}$.\\
\begin{equation}
    \label{eq:fitness-1-obj1}
    \mu(x) = \left\{
    \begin{array}{lcl}
        \sum_{i=1}^n {\mu_i x_i} & & { w(x)\leq B}\\
        B - w(x) & & {w(x) > B}
    \end{array} \right.
\end{equation}
\\
\begin{equation}
    \label{eq:fitness-1-obj2}
    v(x) = \left\{
    \begin{array}{lcl}
        \sum_{i=1}^n {\sigma_i^2x_i} & { w(x)\leq B}\\
        v_{max} +(w(x)-B) & {w(x) >B}
    \end{array} \right.
\end{equation}

When evaluating the fitness of solutions, we need to determine their dominance concerning the different objectives. For two solutions A and B, we say that $A$ dominates $B$ (denoted as $A \succeq B $) \textbf{iff} $\mu(A) \geq \mu(B) \land v(A) \leq v(B)$, and $A$ strongly dominates $B$ (denoted as $A\succ B$) \textbf{iff}  $A\succeq B $ and $\mu(A) > \mu(B) \lor v(A) < v(B)$.  For an infeasible solution that violates the weight capacity constraint ($w(A) > B $), the above formula penalises the two objective values (see Equation \ref{eq:fitness-1-obj1} and \ref{eq:fitness-1-obj2}). This formulation ensures that any feasible solution dominates the infeasible solution instances.

Next, we consider the fitness function $g^{\prime}(x) = (\mu(x), s(x))$. Only the second objective of this fitness function differs from $g^{\prime}(x)$ while the first objective $\mu$ remains the same  and can be defined as given in Equation \ref{eq:fitness-1-obj1}. We denote the maximal standard deviation as $s_{max} = \sum_{i=1}^n {\sigma_i}$ in Equation \ref{eq:fitness-1-obj2-sqrt} which defines the second objective of this fitness function $g^{\prime}(x)$.

\begin{equation}
    \label{eq:fitness-1-obj2-sqrt}
    s(x) = \left\{
    \begin{array}{lcl}
        \sqrt{\sum_{i=1}^n {\sigma_i^2x_i}}  & { w(x)\leq B} \\
        s_{max} +(w(x)-B) & {w(x) >B}.
    \end{array} \right.
\end{equation}

For simple EAs like GSEMO, the final population when optimizing $g(x)$ is equivalent to when optimizing $g^{\prime}(x)$ as the difference between the two is that the second objective of the latter is the square root of the second objective of the former. Therefore it does not change the order of search points of such algorithms. However, it should be noted that some popular EMO algorithms like NSGA-II would encounter a difference in the diversity mechanism when working with $v(x)$ instead of $s(x)$. Therefore, we investigate both $g(x)$ and $g^{\prime}(x)$ in this work.

\subsection{Fitness Function for Profits with Same Dispersion}
\label{subsec:fitness-functions-special}

When the value for $\delta$ is the same for all elements, we can consider the number of items selected in the solution ($|x|_1$) as the second objective.
This objective enables the definition of a fitness function that performs independent of both the confidence level of the solution ($\alpha$) and the uncertainty level of the individual profit values ($\delta$). We denote the fitness function for this scenario as $g^{\prime\prime}(x) = (\mu(x),c(x))$. The first objective of the new fitness function is based on the expectation of the profit of $x$ $\mu(x)$, similar to previous functions and calculated as given in Equation \ref{eq:fitness-1-obj1}. The second objective $c(x)$ is calculated as follows,

\begin{equation}
    \label{eq:fitness-2-obj2}
    c(x) = \left\{
    \begin{array}{lcl}
        \sum_{i=1}^n {x_i} & { w(x)\leq B}\\
        n +(w(x)-B) & {w(x) >B}.
    \end{array} \right.
\end{equation}

\subsection{Estimating the Best Profit Value}
\label{subsec:profit_estimates_mo}
As the final output of the multi-objective evolutionary approaches, we get a set of solutions that gives a trade-off of the objectives independent of the confidence level of the profits. We need to identify the solution with the best profit for a given confidence level $\alpha$ using these final solutions. We can use Equations \ref{eq:profit_cheby} and \ref{eq:profit_cher} from \cite{Aneta_ppsn_22} and calculate the profit of all the solutions in the final population for the required confidence levels ($\alpha$). The solution giving the highest profit value for each $\alpha$ is the best solution for that particular setting.
\subsection{Confidence Level of a Solution's Profit}
\label{subsec:alpha-threshold}
Different solutions in the final population produced by multi-objective approaches become the best solution for different confidence levels. We can identify the confidence level of each solution in the final population for which it gives the highest profit value. 
First, we obtain the $\alpha$ threshold for a pair of solutions such that one becomes better than the other by comparing the profit estimates mentioned in Subsection \ref{subsec:profit_estimates}. Based on the threshold $\alpha$ value between solution pairs in the final population, we define a confidence level interval for each solution such that they give the best profit value for any $\alpha$ in that interval.

\subsubsection{Confidence Level Threshold using Chebyshev's Inequality}
Let the solutions $x$ and $y$ such that $\mu(x)>\mu(y)$ and $v(x)>v(y)$, then we can say that the profit of the solution $x$ is better than $y$ when the minimum required confidence level($\alpha_{x,y}$) is as follows:

\begin{align}
    \label{eq:cheby_alpha_threshold}
    \alpha_{Cheb(x,y)}=\frac{1}{1+(R(x,y))^2} & \text{ s.t.} & R(x,y) = \frac{\mu(x)-\mu(y)}{\sqrt{v(x)} - \sqrt{v(y)}}
\end{align}

For the confidence level $\alpha\geq 1/(1+(R(x,y))^2)$, the solution $X$ will have a better profit value than the solution $Y$.\\
\begin{thm}
Let $0 < \alpha <1, and $ $x$ and $y$ be two feasible solutions such that $\mu(x) > \mu(y)$ and $v(x) > v(y)$, holds. If $\alpha\geq \frac{1}{1+(R(x,y))^2}$ holds such that $R(x,y) = \frac{\mu(x)-\mu(y)}{\sqrt{v(x)} - \sqrt{v(y)}}$ then $\hat{p}_{Cheb}(x, \alpha) \geq \hat{p}_{Cheb}(y, \alpha)$.
\end{thm}

\begin{proof}
    We have
    \begin{align*}
        & &\alpha & \geq \frac{1}{1+(R(x,y))^2} \\
        & \Longleftrightarrow & \alpha + \alpha (R(x,y))^2 & \geq 1 \\
        & \Longleftrightarrow & (R(x,y))^2 & \geq (1- \alpha)/{\alpha}\\
    \end{align*}
    As we assume $0 < \alpha<1$, $\mu(x) > \mu(y)$ and $v(x) > v(y)$, we have $R(x,y)>0$ and $\frac{1- \alpha}{\alpha}>0$.\\
    This implies,
    \begin{align*}
        && R(x,y) & \geq \sqrt{(1- \alpha)/\alpha}\\
        & \Longleftrightarrow & \frac{\mu(x)-\mu(y)}{\sqrt{v(x)} - \sqrt{v(y)}} & \geq \sqrt{(1- \alpha)/\alpha}\\
        & \Longleftrightarrow & \mu(x)-\mu(y) & \geq \sqrt{\frac{1- \alpha}{\alpha}} \cdot \left( \sqrt{v(x)} - \sqrt{v(y)}\right)\\
        & \Longleftrightarrow & \mu(x)- \sqrt{\frac{1- \alpha}{\alpha}} \cdot \sqrt{v(x)} & \geq \mu(y) -  \sqrt{\frac{1- \alpha}{\alpha}} \cdot \sqrt{v(y)}\\
        & \Longleftrightarrow & \hat{p}_{Cheb}(x, \alpha) & \geq \hat{p}_{Cheb}(y, \alpha)
    \end{align*}
\end{proof}

\subsubsection{Confidence Level Threshold using Hoeffding Bound}
Consider the solutions $x$ and $y$ such that $\mu(x)>\mu(y)$ and $v(x)>v(y)$. From the $\hat{p}_{Hoef}$ we can derive the minimum confidence level $\alpha_{x,y}$ for which solution $x$ gives a better profit than $y$ w.r.t. $\hat{p}_{Hoef}$ as follows:
\begin{align}
    \label{eq:cher_alpha_threshold}
    \alpha_{Hoef(x,y)}=e^{-(S(x,y))^2} & \text{ s.t.} & S(x,y) = \frac{\mu(x) - \mu(y)}{\delta\left(\sqrt{2|x|_1} - \sqrt{2|y|_1}\right)}
\end{align}
\begin{thm}
    \label{theo:chern_alpha_threshold}
    Let~$0 < \alpha <1, and $ $x$ and $y$ be two feasible solutions such that $\mu(x) > \mu(y)$ and $v(x) > v(y)$, holds. If $\alpha\geq e^{-(S(x,y))^2}$ holds then $\hat{p}_{Hoef}(x, \alpha) \geq \hat{p}_{Hoef}(y, \alpha)$.
\end{thm}
\begin{proof}
We have,
\begin{align*}
    & & \alpha & \geq e^{-(S(x,y))^2} \\
    &\Longleftrightarrow & e^{(S(x,y))^2} & \geq {1/\alpha} \\
    &\Longleftrightarrow & (S(x,y))^2 & \geq \ln{(1/\alpha)}\\
\end{align*}
As we assume $0<\alpha<1$, $\mu(x)>\mu(y)$ and $v(x)>v(y)$, we have $S(x,y)>0$ and $\ln\frac{1}{\alpha} > 0$.\\
This implies, 
\begin{align*}
    &&S(x,y) & \geq \sqrt{\ln(1/\alpha)} \\
    &\Longleftrightarrow &  \frac{\mu(x) - \mu(y)}{\delta\left(\sqrt{2|x|_1} - \sqrt{2|y|_1}\right)} & \geq 
        \sqrt{\ln(1/\alpha)} \\
    &\Longleftrightarrow &  \mu(x) - \mu(y) & \geq
        \delta \cdot \sqrt{2 \ln \left( \frac{1}{\alpha} \right) } \left( \sqrt{|x|_1} - \sqrt{|y|_1} \right) \\
    &\Longleftrightarrow & \mu(x) - \delta \cdot \sqrt{\ln\left(\frac{1}{\alpha}\right)\cdot2\cdot|x|_1} & \geq 
        \mu(y) - \delta \cdot \sqrt{\ln\left(\frac{1}{\alpha}\right)\cdot2\cdot|y|_1}\\
    &\Longleftrightarrow & \hat{p}_{Hoef}(x, \alpha) & \geq 
        \hat{p}_{Hoef}(y, \alpha)
\end{align*}
\end{proof}

\subsubsection{Confidence Level Interval for Solutions in the Final Population}
We can use the $\alpha$ threshold value in the previous subsections, to introduce the confidence level range for solutions in a population as follows. As the $\alpha$ threshold we can use either $\alpha_{Cheb(x,y)}$ or $\alpha_{Hoef(x,y)}$. Despite the specific equation to estimate the $\alpha$ threshold, we can introduce the $\alpha$ interval for a solution. 

First, we sort all the solutions in the given population P as $\{x_1, ... x_{n}\}$ such that $\mu(x_1)\geq \mu(x_2) \geq ... \geq \mu(x_{n})$. Then, we can define a $(n+1)\times (n+1)$ symmetric matrix of confidence level thresholds as:

\begin{eqnarray}
    \label{eq:alpha-matrix}
    \alpha_{(i,j)} =
    \left\{
        \begin{array}{lcl}
            1 & i = 0 \text{ or }j = 0\\
            \alpha_{Cheb(i,j)} \text{ or } \alpha_{Hoef(i,j)} & i=1,...,n
        \end{array} 
    \right.
\end{eqnarray}

Finally, using the confidence level threshold values, we can get the confidence level range for solution k as given in Equation \ref{eq:alpha-range}.
If there exists a valid $\alpha$ interval for a particular solution, then for the $\alpha$ values in that interval, that solution will give the best profit value. If the $\alpha$ interval is empty for a particular solution, then that solution won't be the best solution for any $\alpha$ value.
\begin{equation}
    \label{eq:alpha-range}
    \min_{j=0}^{k-1} \alpha_{i,k} \leq \alpha_k \leq \max_{j=k}^{n-1} \alpha_{i,j}
\end{equation}
\section{Algorithms}
\label{sec:algorithms}
\begin{algorithm}[!t]
\caption{GSEMO}
    \begin{algorithmic}[1]
        \STATE Choose $x \in \{0,1\}^n $ uniformly at random ;
        \STATE $S\leftarrow \{x\}$;
        \WHILE{stopping criterion not met}
            \STATE choose $x\in S$ uniformly at random;
            \STATE $y\leftarrow$ flip each bit of $x$ independently with probability of $\frac{1}{n}$;
            \IF{($\not\exists w \in S: w \succ y$)}
                \STATE $S \leftarrow (S \cup \{y\})\backslash \{z\in S \mid y \succeq z\}$ ;
            \ENDIF
        \ENDWHILE
    \end{algorithmic}
\label{alg:GSEMO}
\end{algorithm}

In this study, we consider two widely used multi-objective EAs: GSEMO and NSGA-II. GSEMO is the most basic EA that addresses multi-objective optimization problems. It has been proven to be effective in solving chance constrained multi-objective optimization problems in many studies \cite{Aneta_ppsn_22}. The steps of GSEMO are given in Algorithm \ref{alg:GSEMO}.

The population S initially contains a solution that is generated randomly. Over the iterations, a parent solution $x$ is selected uniformly at random from S and an offspring solution $y$ is generated by flipping each bit of $x$ with a probability of $\frac{1}{n}$. If $y$ is not dominated by any of the existing solutions in S, it is added to S replacing all the existing solutions in S that are dominated by y. This guarantees at the end of any iteration, population S will contain a set of non-dominating solutions that are equally efficient with respect to the given objective functions. 

NSGA-II is the most prominent multi-objective EA that focuses on diversity by finding near optimal solutions \cite{DebAPM02_nsga2}. If we consider one iteration of the evolutionary process of NSGA-II, it creates an offspring population from the parent population. First, we use the binary tournament for the selection of two parent solutions and apply single point crossover to generate two offspring solutions which are mutated by flipping each bit with a probability of $\frac{1}{n}$. When the offspring population is full, all the solutions from parent and offspring solutions are considered together and divided into non-dominating solution fronts (i.e., the solutions in a front do not dominate other solutions in it). Starting with the front with the lowest rank, these are added to the new population until it reaches capacity. If only a part of a front is to be taken into the new population, crowding distance-based sorting is used to decide the selection. This improves the better spread (diversity) among the solutions. In one iteration of this algorithm, a new fitness evaluation equal to the size of the offspring population is considered. The number of iterations of the algorithms depends on the offspring population size and the maximum number of evaluations as considered in the experimental setup.
\section{Filtering for Problems with Chance Constraints}
\label{subsec:filtering}
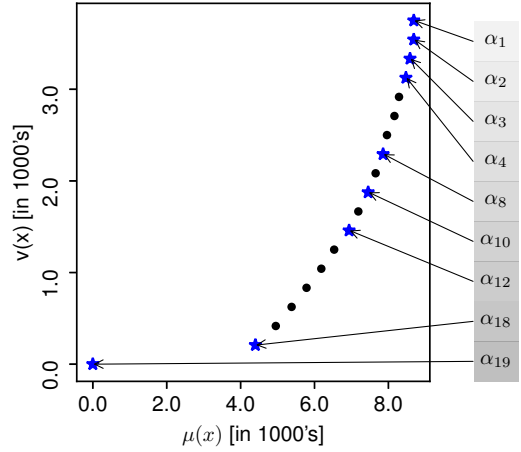
\begin{figure}[t]
    \begin{center}
       \scalebox{.8}{\input{plot1.pgf}}
    \end{center}
    \caption{A sample of a final population from GSEMO}
    \label{fig:plot1}
\end{figure}
Here we propose a filtering method for GSEMO to improve its outcome by removing solutions that do not have a valid $\alpha$ interval from the interim populations. The final population of GSEMO contains solutions that do not give the best profit value for any probability value for $\alpha$. Such solutions do not add value to the optimization goal of finding the best solutions with given confidence levels. For instance, Figure \ref{fig:plot1} presents a final population from GSEMO using 10 million fitness evaluation of $g(x)$ on uncorr-100: an instance used in the experiments. The plot shows $\mu(x)$ versus $v(x)$ for all solutions in the final population. Blue star markers indicate that the solution has a valid confidence level interval. These solutions' $\alpha$ intervals compose the complete probability range [0.0,1.0]. At the end of the optimization process, interest is in solutions with these blue markers. It is one of these solutions that become the best solution for a given $\alpha$ value. On the contrary, other solutions marked in black do not become the best solution for any confidence level.

The filtering method removes these solutions that do not become the best solution for any confidence level. It is applied to the interim populations of GSEMO regularly after a certain amount of fitness evaluations. This process removes the solutions from interim populations, considering the $\alpha$ intervals they represent according to Equation \ref{eq:alpha-range}. As the filtering method keeps only the solutions with valid $\alpha$ interval, it increases the change of the evolutionary optimization to improve on these solutions. Therefore, this method helps to improve the quality of the solutions in the final population.

The steps of the filtering method are given in Algorithm \ref{alg:filtering-2}. It takes the population $P_0$ as the input, which can be either the final population or an interim population created during the execution of GSEMO. Population $P_0$ needs solutions in the decreasing order of the $\mu(x)$.
For each solution $x^k$, we consider $\alpha_{i,k}$ and select the interval for confidence level, using the inequality given in Equation \ref{eq:alpha-range}. 
The solution $x^k$ is added to the resulting population $P_1$ iff the interval for $\alpha_k$ is non-empty. 
\begin{algorithm}[t]
    \caption{Filtering Method}
    \label{alg:filtering-2}
    \begin{algorithmic}[1]
        \item Input: Population $P_0 = {x^1, ..., x^n}$ ordered by the decreasing order of $\mu(x^i)$;
        \STATE set $k \leftarrow 1$;
        \WHILE{$k \leq n$}
            \STATE set $upper \leftarrow \min_{j=0}^{k-1} \alpha_{i,k}$
            \STATE set $lower \leftarrow \max_{j=k}^{n-1} \alpha_{i,j}$
            \IF{$upper \geq lower $}
                \STATE $P_1 \cup \{x^k\}$
            \ENDIF
            \STATE set $k \leftarrow k + 1$
        \ENDWHILE
        \RETURN $P_1$
    \end{algorithmic}
\end{algorithm}

\section{Experiments}
\label{sec:experiments}


In this work, we evaluate the fitness functions introduced previously in different chance constraint settings using multi-objective EAs on different benchmarks; and we use multiple experiments for that. Here we present the experimental settings and the results of these experiments.

\begin{landscape}
\input{tables/tab1_cheb_fixed_delta.tex}
\end{landscape}
\subsection{Experimental Setup}
\label{subsec:experimental-setup}
For experimental investigations, we use the same benchmarks as the ones that are used in \cite{Aneta_ppsn_22}. The set of benchmarks includes three correlated instances and three bounded strongly correlated instances with the numbers of knapsack items $n \in \{100,300,500\}$. We consider that the profits of the knapsack items have a uniform distribution, such that the profit of element $i$ is chosen uniformly at random as $p_i \in \{\mu_i-\delta,\mu_i+\delta\}$. This allows to use both $\hat{p}_{Cheb}$ and $\hat{p}_{Hoef}$ when the profits have the same uncertainty level ($\delta$). The experimental investigation covers two uncertainty levels for each benchmark, $\delta \in \{25,50\}$.
Additionally, we consider the scenario where the profits have different dispersion such that each item $i$ has an uncertainty level $\delta_i$, which is chosen uniformly at random as $\delta_i \in [0.0, \mu_i]$. The benchmarks with different uncertainties are considered only with $\hat{p}_{Cheb}$ since $\hat{p}_{Hoef}$ requires the same uncertainty level for all elements.

This study mainly considers three algorithms: two well-known multi-objective evolutionary algorithms GSEMO and NSGA-II and the third is GSEMO with the filtering method introduced in Section \ref{sec:algorithms}. The latter is referred to as GSEMO+Filtering hereafter in this paper. These algorithms are combined with fitness functions as appropriate, to use in the experiments. For the benchmarks with fixed uncertainties, using GSEMO or GSEMO+Filtering with any fitness function ($g(x)$ or $g^{\prime}(x)$ or $g^{\prime\prime}(x)$) will produce equivalent results as the final populations. Therefore, GSEMO and GSEMO+Filtering are only considered with the fitness evaluation $g(x)$. The fitness function $g^{\prime\prime}(x)$ considers the number of items selected in the solution for the scenario where profits have the same dispersion. Therefore, we do not consider algorithms with that fitness function for the benchmarks with different uncertainties for profits. Only GSEMO and GSEMO+Filtering with $g(x)$, and NSGA-II with $g(x)$ and $g^{\prime}(x)$ are considered for those benchmarks. 

Every algorithm considers the 10 million fitness evaluation and produces a population of solutions that gives a trade-off concerning the objectives used for fitness evaluation. The quality of the output of these methods is evaluated by analyzing the best profit value for different confidence levels considering $\alpha \in \{0.1, 0.01, 0.001\}$. Each method generates a final population independent of $\alpha$, and we select the best solution from that population for different $\alpha$ values using profit estimations $\hat{p}_{Cheb}$ or $\hat{p}_{Hoef}$ as applicable. The results summarise the best profit value given by 30 experimental results for each $\alpha$. This summary requires running the method 30 times for each $\delta$ for benchmarks with the same profit dispersion and 30 times for benchmarks with different profit dispersion. However, when using $g^{\prime\prime}(x)$, as algorithms run independent of $\delta$ value, it is possible to get the best profit values for different $\delta$ from the same final population. 

Finally, we test for the statistical significance validity of the results using the Kruskal-Wallis test with 95\% confidence with the Bonferroni post-hoc statistical procedure. The statistical comparison is indicated as $X^{+}$ or $X^{-}$ to indicate that the method in the column outperforms $X$ or vice versa. If there is no significant difference between the two methods, respective numbers do not appear. For each method, the summary of the best profit values is given as mean, std and stat, which represent the mean and standard deviation of the results and statistical comparison with the corresponding results from other methods, respectively.
\subsection{Results}
\label{subsec:results}
\input{tables/tab2_cher_fixed_delta.tex}
\input{tables/tab5_rand_del.tex}
Table \ref{tab:cheb-same-delta} and \ref{tab:cher-same-delta} present the results for the benchmarks with fixed uncertainty levels profits of elements. According to the mean values GSEMO with the filtering method outperforms other methods in most of the settings in both tables. The statistical comparisons give more insights into the performance of the methods when applied to each instance. Results for uncorr-100 in Table \ref{tab:cheb-same-delta} show that GSEMO+Filtering performs the better than other four methods, and GSEMO performs better than NSGA-II with $g^{\prime\prime}(x)$. For smaller confidence levels $\alpha=0.001$, NSGA-II with $g(x)$ outperforms that with $g^{\prime\prime}(x)$. Strong-100 instance gets very different results for $\alpha=0.001$ and $\delta=50$, than other cases of the same instance. There, GSEMO with $g(x)$ and NSGA-II with $g(x)$ and $g^\prime(x)$ perform well and GSEMO+Filtering gives the lowest result. However, the second method performs well in other $\alpha$ and $\delta$ settings and NSGA-II with $g(x)$ and $g^\prime(x)$ also perform similarly in certain settings.

\sloppy 

For all settings of uncorr-300, GSEMO+Filtering outperforms the other four methods while performing similarly to each other. For strong-300 instance, GSEMO gives lower results than GSEMO+Filtering and NSGA-II with $g(x)$ and $g^{\prime\prime}(x)$. Also, NSGA-II with $g(x)$ produces results as good as GSEMO+Filtering for $\alpha=0.001$. For uncorr-500, GSEMO+Filtering gives the best results and GSEMO outperforms most of the NSGA-II results. However, when using $g^{\prime\prime}(x)$ with NSGA-II on this instance, results show equal performance except for $\alpha=\{0.1, 0.01\}$ confidence levels when considering uncertainty level $\delta=50$. Experiments on strong-500 also get the best results from GSEMO+Filtering. However, GSEMO is outperformed by other methods for all settings considered for this instance. In comparison, the NSGA-II gives the second best results across settings when $\delta=50$ and $\alpha=0.001$ NSGA-II methods also perform as well as GSEMO+Filtering.

Table \ref{tab:cher-same-delta} gives the results from the experiments that use $\hat{p}_{Cheb}$ to estimate the profit values of solutions. uncorr-100, strong-100, and uncorr-300 results are highest when using GSEMO+Filtering. It outperforms all the other methods except NSGA-II with $g^\prime(x)$ on strong-100 instance for uncertainty level $\delta=50$. Experiments on strong-300 instance show that NSGA-II with $g^{\prime\prime}(x)$ performs equally as the GSEMO+Filtering when considering a lower uncertainty level (25) for the knapsack items. GSEMO and NSGA-II with $g^{\prime}(x)$ methods give the lowest results for this instance. On the contrary, GSEMO performs better than most NSGA-II methods for uncorr-500. When using $g^{\prime\prime}(x)$, NSGA-II performs equally as GSEMO for lower uncertainty level value 25. For all $\alpha$ and $\delta$ values, NSGA-II methods are outperformed by GSEMO+Filtering in the experiments on strong-500 in Table \ref{tab:cher-same-delta}. However, NSGA-II methods perform better than GSEMO on that benchmark. 

GSEMO+Filtering performs significantly well when applied to benchmarks with the same dispersion for profits. The filtering method allows the interim populations to contain more solutions that yield a valid confidence level interval. Therefore, improving upon these solutions eventually gives better results in the final outcome. Considering NSGA-II results, $g(x)$ and $g^{\prime}(x)$ tends to produce better results than $g^{\prime\prime}(x)$. This can be due to the fact the crowding distance assignment is better when considering the variance or the standard deviation of solutions' profits. 

We can compare the results for benchmarks with the same dispersion for profits (given in Table \ref{tab:cheb-same-delta} and \ref{tab:cher-same-delta}) with the previous work in \cite{Aneta_ppsn_22} as it also considers the same benchmarks and similar experimental setup. The experimental settings are the same in both works except for the number of fitness evaluations the algorithms consider. For each $\alpha$ and $\delta$ value, methods in \cite{Aneta_ppsn_22} run for one million fitness evaluations. On the contrary, we run the multi-objective methods on benchmarks for each $\delta$ value for 10 million fitness evaluations which yield results for all $\alpha$ values from the same algorithm output. The results show that for Chebyshev results 14 out of the 36 settings, the highest mean profit values given in Table \ref{tab:cheb-same-delta} outperform all the three methods used in \cite{Aneta_ppsn_22} and for Hoeffding results in 15 out of the 36 settings gets better profits according to Table \ref{tab:cher-same-delta}. Generally, most of these cases are from experiments on uncorrelated benchmarks. For the cases where the new methods perform better on bounded strongly correlated instances, it is for higher uncertainty level $\delta=50$ and lower confidence values like $\alpha=\{0.01,0.01\}$.

Table \ref{tab:rand_del} presents results for benchmarks with different dispersion for profits of elements. The highest mean values reported for each case show, for smaller instances, GSEMO and NSGA-II give the highest mean profits, and for instances with 300 or 500 elements, GSEMO+Filtering with $g^\prime(x)$ and NSGA-II give the highest mean profit. Although the highest mean values can be identified from different methods, the results are similar in most of the settings. Based on the statistical comparisons, we can see that smaller instances: strong-100 and uncorr-100, do not significantly differ between each method's results. Results from other instances show that for $\alpha=0.1$ GSEMO+Filtering and NSGA-II methods outperform the GSEMO. In addition, strong-500 instance shows better results from GSEMO+Filtering with $g^\prime(x)$ and NSGA-II with both fitness functions for $\alpha=0.01$. However, in other settings with $\alpha=\{0.01,0.001\}$, all methods show similar results for the instances with 300 and 500 items.

Compared to the benchmarks with the same dispersion for profits, the ones with different dispersion can give higher variance for the profit of some elements. Therefore, it is crucial to have a good spread of the solutions in the final population as more solutions tend to give negative profit values when considering certain confidence levels. NSGA-II using crowding distance based sorting appears to achieve this when use with selected objectives in the fitness evaluations. In comparison, GSEMO+Filtering is also able to achieve similarly good results by ignoring the solutions without a valid $\alpha$ interval and eventually improving the final population. 

\section{Conclusion}
\label{sec:conclusion}
This paper explores multi-objective evolutionary approaches to solve the profit chance constrained KP. We introduce fitness evaluations for EAs to cater for this problem. These fitness functions can evaluate the solutions irrespective of the required confidence level of the solutions. Therefore, the outcome of EAs gives us a population that includes solutions giving the best profit value for different confidence levels. So it is unnecessary to decide on the required confidence level before executing the algorithms, and no need to have multiple executions to investigate solutions for different confidence levels. After considering the available solutions and risks associated with their profits, it is possible to make more informed decisions on what solutions to implement for the problem instance. Furthermore, we introduce a filtering method, which is applied at regular intervals of fitness evaluations. It keeps only the solutions with a valid $\alpha$ interval in the interim populations, enabling the new offspring solutions in the next generations to improve upon these solutions. The performance of these methods is evident in the experimental investigations.

\section*{Acknowledgements}
This work has been supported by the Australian Research Council (ARC) through grant FT200100536, and by the South Australian Government through the Research Consortium "Unlocking Complex Resources through Lean Processing". This work was also supported with supercomputing resources provided by the Phoenix HPC service at the University of Adelaide.


\bibliographystyle{ACM-Reference-Format}
\bibliography{main}



\end{document}

%% file: plot1.pgf
\begingroup%
\makeatletter%
\begin{pgfpicture}%
\pgfpathrectangle{\pgfpointorigin}{\pgfqpoint{3.500000in}{3.000000in}}%
\pgfusepath{use as bounding box, clip}%
\begin{pgfscope}%
\pgfsetbuttcap%
\pgfsetmiterjoin%
\definecolor{currentfill}{rgb}{1.000000,1.000000,1.000000}%
\pgfsetfillcolor{currentfill}%
\pgfsetlinewidth{0.000000pt}%
\definecolor{currentstroke}{rgb}{1.000000,1.000000,1.000000}%
\pgfsetstrokecolor{currentstroke}%
\pgfsetdash{}{0pt}%
\pgfpathmoveto{\pgfqpoint{0.000000in}{0.000000in}}%
\pgfpathlineto{\pgfqpoint{3.500000in}{0.000000in}}%
\pgfpathlineto{\pgfqpoint{3.500000in}{3.000000in}}%
\pgfpathlineto{\pgfqpoint{0.000000in}{3.000000in}}%
\pgfpathlineto{\pgfqpoint{0.000000in}{0.000000in}}%
\pgfpathclose%
\pgfusepath{fill}%
\end{pgfscope}%
\begin{pgfscope}%
\pgfsetbuttcap%
\pgfsetmiterjoin%
\definecolor{currentfill}{rgb}{1.000000,1.000000,1.000000}%
\pgfsetfillcolor{currentfill}%
\pgfsetlinewidth{0.000000pt}%
\definecolor{currentstroke}{rgb}{0.000000,0.000000,0.000000}%
\pgfsetstrokecolor{currentstroke}%
\pgfsetstrokeopacity{0.000000}%
\pgfsetdash{}{0pt}%
\pgfpathmoveto{\pgfqpoint{0.525000in}{0.450000in}}%
\pgfpathlineto{\pgfqpoint{2.835000in}{0.450000in}}%
\pgfpathlineto{\pgfqpoint{2.835000in}{2.925000in}}%
\pgfpathlineto{\pgfqpoint{0.525000in}{2.925000in}}%
\pgfpathlineto{\pgfqpoint{0.525000in}{0.450000in}}%
\pgfpathclose%
\pgfusepath{fill}%
\end{pgfscope}%
\begin{pgfscope}%
\pgfsetbuttcap%
\pgfsetroundjoin%
\definecolor{currentfill}{rgb}{0.000000,0.000000,0.000000}%
\pgfsetfillcolor{currentfill}%
\pgfsetlinewidth{0.803000pt}%
\definecolor{currentstroke}{rgb}{0.000000,0.000000,0.000000}%
\pgfsetstrokecolor{currentstroke}%
\pgfsetdash{}{0pt}%
\pgfsys@defobject{currentmarker}{\pgfqpoint{0.000000in}{-0.048611in}}{\pgfqpoint{0.000000in}{0.000000in}}{%
\pgfpathmoveto{\pgfqpoint{0.000000in}{0.000000in}}%
\pgfpathlineto{\pgfqpoint{0.000000in}{-0.048611in}}%
\pgfusepath{stroke,fill}%
}%
\begin{pgfscope}%
\pgfsys@transformshift{0.630000in}{0.450000in}%
\pgfsys@useobject{currentmarker}{}%
\end{pgfscope}%
\end{pgfscope}%
\begin{pgfscope}%
\definecolor{textcolor}{rgb}{0.000000,0.000000,0.000000}%
\pgfsetstrokecolor{textcolor}%
\pgfsetfillcolor{textcolor}%
\pgftext[x=0.630000in,y=0.352778in,,top]{\color{textcolor}\sffamily\fontsize{10.000000}{12.000000}\selectfont 0.0}%
\end{pgfscope}%
\begin{pgfscope}%
\pgfsetbuttcap%
\pgfsetroundjoin%
\definecolor{currentfill}{rgb}{0.000000,0.000000,0.000000}%
\pgfsetfillcolor{currentfill}%
\pgfsetlinewidth{0.803000pt}%
\definecolor{currentstroke}{rgb}{0.000000,0.000000,0.000000}%
\pgfsetstrokecolor{currentstroke}%
\pgfsetdash{}{0pt}%
\pgfsys@defobject{currentmarker}{\pgfqpoint{0.000000in}{-0.048611in}}{\pgfqpoint{0.000000in}{0.000000in}}{%
\pgfpathmoveto{\pgfqpoint{0.000000in}{0.000000in}}%
\pgfpathlineto{\pgfqpoint{0.000000in}{-0.048611in}}%
\pgfusepath{stroke,fill}%
}%
\begin{pgfscope}%
\pgfsys@transformshift{1.113481in}{0.450000in}%
\pgfsys@useobject{currentmarker}{}%
\end{pgfscope}%
\end{pgfscope}%
\begin{pgfscope}%
\definecolor{textcolor}{rgb}{0.000000,0.000000,0.000000}%
\pgfsetstrokecolor{textcolor}%
\pgfsetfillcolor{textcolor}%
\pgftext[x=1.113481in,y=0.352778in,,top]{\color{textcolor}\sffamily\fontsize{10.000000}{12.000000}\selectfont 2.0}%
\end{pgfscope}%
\begin{pgfscope}%
\pgfsetbuttcap%
\pgfsetroundjoin%
\definecolor{currentfill}{rgb}{0.000000,0.000000,0.000000}%
\pgfsetfillcolor{currentfill}%
\pgfsetlinewidth{0.803000pt}%
\definecolor{currentstroke}{rgb}{0.000000,0.000000,0.000000}%
\pgfsetstrokecolor{currentstroke}%
\pgfsetdash{}{0pt}%
\pgfsys@defobject{currentmarker}{\pgfqpoint{0.000000in}{-0.048611in}}{\pgfqpoint{0.000000in}{0.000000in}}{%
\pgfpathmoveto{\pgfqpoint{0.000000in}{0.000000in}}%
\pgfpathlineto{\pgfqpoint{0.000000in}{-0.048611in}}%
\pgfusepath{stroke,fill}%
}%
\begin{pgfscope}%
\pgfsys@transformshift{1.596962in}{0.450000in}%
\pgfsys@useobject{currentmarker}{}%
\end{pgfscope}%
\end{pgfscope}%
\begin{pgfscope}%
\definecolor{textcolor}{rgb}{0.000000,0.000000,0.000000}%
\pgfsetstrokecolor{textcolor}%
\pgfsetfillcolor{textcolor}%
\pgftext[x=1.596962in,y=0.352778in,,top]{\color{textcolor}\sffamily\fontsize{10.000000}{12.000000}\selectfont 4.0}%
\end{pgfscope}%
\begin{pgfscope}%
\pgfsetbuttcap%
\pgfsetroundjoin%
\definecolor{currentfill}{rgb}{0.000000,0.000000,0.000000}%
\pgfsetfillcolor{currentfill}%
\pgfsetlinewidth{0.803000pt}%
\definecolor{currentstroke}{rgb}{0.000000,0.000000,0.000000}%
\pgfsetstrokecolor{currentstroke}%
\pgfsetdash{}{0pt}%
\pgfsys@defobject{currentmarker}{\pgfqpoint{0.000000in}{-0.048611in}}{\pgfqpoint{0.000000in}{0.000000in}}{%
\pgfpathmoveto{\pgfqpoint{0.000000in}{0.000000in}}%
\pgfpathlineto{\pgfqpoint{0.000000in}{-0.048611in}}%
\pgfusepath{stroke,fill}%
}%
\begin{pgfscope}%
\pgfsys@transformshift{2.080443in}{0.450000in}%
\pgfsys@useobject{currentmarker}{}%
\end{pgfscope}%
\end{pgfscope}%
\begin{pgfscope}%
\definecolor{textcolor}{rgb}{0.000000,0.000000,0.000000}%
\pgfsetstrokecolor{textcolor}%
\pgfsetfillcolor{textcolor}%
\pgftext[x=2.080443in,y=0.352778in,,top]{\color{textcolor}\sffamily\fontsize{10.000000}{12.000000}\selectfont 6.0}%
\end{pgfscope}%
\begin{pgfscope}%
\pgfsetbuttcap%
\pgfsetroundjoin%
\definecolor{currentfill}{rgb}{0.000000,0.000000,0.000000}%
\pgfsetfillcolor{currentfill}%
\pgfsetlinewidth{0.803000pt}%
\definecolor{currentstroke}{rgb}{0.000000,0.000000,0.000000}%
\pgfsetstrokecolor{currentstroke}%
\pgfsetdash{}{0pt}%
\pgfsys@defobject{currentmarker}{\pgfqpoint{0.000000in}{-0.048611in}}{\pgfqpoint{0.000000in}{0.000000in}}{%
\pgfpathmoveto{\pgfqpoint{0.000000in}{0.000000in}}%
\pgfpathlineto{\pgfqpoint{0.000000in}{-0.048611in}}%
\pgfusepath{stroke,fill}%
}%
\begin{pgfscope}%
\pgfsys@transformshift{2.563924in}{0.450000in}%
\pgfsys@useobject{currentmarker}{}%
\end{pgfscope}%
\end{pgfscope}%
\begin{pgfscope}%
\definecolor{textcolor}{rgb}{0.000000,0.000000,0.000000}%
\pgfsetstrokecolor{textcolor}%
\pgfsetfillcolor{textcolor}%
\pgftext[x=2.563924in,y=0.352778in,,top]{\color{textcolor}\sffamily\fontsize{10.000000}{12.000000}\selectfont 8.0}%
\end{pgfscope}%
\begin{pgfscope}%
\definecolor{textcolor}{rgb}{0.000000,0.000000,0.000000}%
\pgfsetstrokecolor{textcolor}%
\pgfsetfillcolor{textcolor}%
\pgftext[x=1.680000in,y=0.162809in,,top]{\color{textcolor}\sffamily\fontsize{10.000000}{12.000000}\selectfont \(\displaystyle \mu(x)\) [in 1000's]}%
\end{pgfscope}%
\begin{pgfscope}%
\pgfsetbuttcap%
\pgfsetroundjoin%
\definecolor{currentfill}{rgb}{0.000000,0.000000,0.000000}%
\pgfsetfillcolor{currentfill}%
\pgfsetlinewidth{0.803000pt}%
\definecolor{currentstroke}{rgb}{0.000000,0.000000,0.000000}%
\pgfsetstrokecolor{currentstroke}%
\pgfsetdash{}{0pt}%
\pgfsys@defobject{currentmarker}{\pgfqpoint{-0.048611in}{0.000000in}}{\pgfqpoint{-0.000000in}{0.000000in}}{%
\pgfpathmoveto{\pgfqpoint{-0.000000in}{0.000000in}}%
\pgfpathlineto{\pgfqpoint{-0.048611in}{0.000000in}}%
\pgfusepath{stroke,fill}%
}%
\begin{pgfscope}%
\pgfsys@transformshift{0.525000in}{0.562500in}%
\pgfsys@useobject{currentmarker}{}%
\end{pgfscope}%
\end{pgfscope}%
\begin{pgfscope}%
\definecolor{textcolor}{rgb}{0.000000,0.000000,0.000000}%
\pgfsetstrokecolor{textcolor}%
\pgfsetfillcolor{textcolor}%
\pgftext[x=0.398888in, y=0.394382in, left, base,rotate=90.000000]{\color{textcolor}\sffamily\fontsize{10.000000}{12.000000}\selectfont 0.0}%
\end{pgfscope}%
\begin{pgfscope}%
\pgfsetbuttcap%
\pgfsetroundjoin%
\definecolor{currentfill}{rgb}{0.000000,0.000000,0.000000}%
\pgfsetfillcolor{currentfill}%
\pgfsetlinewidth{0.803000pt}%
\definecolor{currentstroke}{rgb}{0.000000,0.000000,0.000000}%
\pgfsetstrokecolor{currentstroke}%
\pgfsetdash{}{0pt}%
\pgfsys@defobject{currentmarker}{\pgfqpoint{-0.048611in}{0.000000in}}{\pgfqpoint{-0.000000in}{0.000000in}}{%
\pgfpathmoveto{\pgfqpoint{-0.000000in}{0.000000in}}%
\pgfpathlineto{\pgfqpoint{-0.048611in}{0.000000in}}%
\pgfusepath{stroke,fill}%
}%
\begin{pgfscope}%
\pgfsys@transformshift{0.525000in}{1.162500in}%
\pgfsys@useobject{currentmarker}{}%
\end{pgfscope}%
\end{pgfscope}%
\begin{pgfscope}%
\definecolor{textcolor}{rgb}{0.000000,0.000000,0.000000}%
\pgfsetstrokecolor{textcolor}%
\pgfsetfillcolor{textcolor}%
\pgftext[x=0.398888in, y=0.994382in, left, base,rotate=90.000000]{\color{textcolor}\sffamily\fontsize{10.000000}{12.000000}\selectfont 1.0}%
\end{pgfscope}%
\begin{pgfscope}%
\pgfsetbuttcap%
\pgfsetroundjoin%
\definecolor{currentfill}{rgb}{0.000000,0.000000,0.000000}%
\pgfsetfillcolor{currentfill}%
\pgfsetlinewidth{0.803000pt}%
\definecolor{currentstroke}{rgb}{0.000000,0.000000,0.000000}%
\pgfsetstrokecolor{currentstroke}%
\pgfsetdash{}{0pt}%
\pgfsys@defobject{currentmarker}{\pgfqpoint{-0.048611in}{0.000000in}}{\pgfqpoint{-0.000000in}{0.000000in}}{%
\pgfpathmoveto{\pgfqpoint{-0.000000in}{0.000000in}}%
\pgfpathlineto{\pgfqpoint{-0.048611in}{0.000000in}}%
\pgfusepath{stroke,fill}%
}%
\begin{pgfscope}%
\pgfsys@transformshift{0.525000in}{1.762500in}%
\pgfsys@useobject{currentmarker}{}%
\end{pgfscope}%
\end{pgfscope}%
\begin{pgfscope}%
\definecolor{textcolor}{rgb}{0.000000,0.000000,0.000000}%
\pgfsetstrokecolor{textcolor}%
\pgfsetfillcolor{textcolor}%
\pgftext[x=0.398888in, y=1.594382in, left, base,rotate=90.000000]{\color{textcolor}\sffamily\fontsize{10.000000}{12.000000}\selectfont 2.0}%
\end{pgfscope}%
\begin{pgfscope}%
\pgfsetbuttcap%
\pgfsetroundjoin%
\definecolor{currentfill}{rgb}{0.000000,0.000000,0.000000}%
\pgfsetfillcolor{currentfill}%
\pgfsetlinewidth{0.803000pt}%
\definecolor{currentstroke}{rgb}{0.000000,0.000000,0.000000}%
\pgfsetstrokecolor{currentstroke}%
\pgfsetdash{}{0pt}%
\pgfsys@defobject{currentmarker}{\pgfqpoint{-0.048611in}{0.000000in}}{\pgfqpoint{-0.000000in}{0.000000in}}{%
\pgfpathmoveto{\pgfqpoint{-0.000000in}{0.000000in}}%
\pgfpathlineto{\pgfqpoint{-0.048611in}{0.000000in}}%
\pgfusepath{stroke,fill}%
}%
\begin{pgfscope}%
\pgfsys@transformshift{0.525000in}{2.362500in}%
\pgfsys@useobject{currentmarker}{}%
\end{pgfscope}%
\end{pgfscope}%
\begin{pgfscope}%
\definecolor{textcolor}{rgb}{0.000000,0.000000,0.000000}%
\pgfsetstrokecolor{textcolor}%
\pgfsetfillcolor{textcolor}%
\pgftext[x=0.398888in, y=2.194382in, left, base,rotate=90.000000]{\color{textcolor}\sffamily\fontsize{10.000000}{12.000000}\selectfont 3.0}%
\end{pgfscope}%
\begin{pgfscope}%
\definecolor{textcolor}{rgb}{0.000000,0.000000,0.000000}%
\pgfsetstrokecolor{textcolor}%
\pgfsetfillcolor{textcolor}%
\pgftext[x=0.237809in,y=1.687500in,,bottom,rotate=90.000000]{\color{textcolor}\sffamily\fontsize{10.000000}{12.000000}\selectfont v(x) [in 1000's]}%
\end{pgfscope}%
\begin{pgfscope}%
\pgfpathrectangle{\pgfqpoint{0.525000in}{0.450000in}}{\pgfqpoint{2.310000in}{2.475000in}}%
\pgfusepath{clip}%
\pgfsetbuttcap%
\pgfsetroundjoin%
\definecolor{currentfill}{rgb}{0.000000,0.000000,0.000000}%
\pgfsetfillcolor{currentfill}%
\pgfsetlinewidth{1.003750pt}%
\definecolor{currentstroke}{rgb}{0.000000,0.000000,0.000000}%
\pgfsetstrokecolor{currentstroke}%
\pgfsetdash{}{0pt}%
\pgfsys@defobject{currentmarker}{\pgfqpoint{-0.020833in}{-0.020833in}}{\pgfqpoint{0.020833in}{0.020833in}}{%
\pgfpathmoveto{\pgfqpoint{0.000000in}{-0.020833in}}%
\pgfpathcurveto{\pgfqpoint{0.005525in}{-0.020833in}}{\pgfqpoint{0.010825in}{-0.018638in}}{\pgfqpoint{0.014731in}{-0.014731in}}%
\pgfpathcurveto{\pgfqpoint{0.018638in}{-0.010825in}}{\pgfqpoint{0.020833in}{-0.005525in}}{\pgfqpoint{0.020833in}{0.000000in}}%
\pgfpathcurveto{\pgfqpoint{0.020833in}{0.005525in}}{\pgfqpoint{0.018638in}{0.010825in}}{\pgfqpoint{0.014731in}{0.014731in}}%
\pgfpathcurveto{\pgfqpoint{0.010825in}{0.018638in}}{\pgfqpoint{0.005525in}{0.020833in}}{\pgfqpoint{0.000000in}{0.020833in}}%
\pgfpathcurveto{\pgfqpoint{-0.005525in}{0.020833in}}{\pgfqpoint{-0.010825in}{0.018638in}}{\pgfqpoint{-0.014731in}{0.014731in}}%
\pgfpathcurveto{\pgfqpoint{-0.018638in}{0.010825in}}{\pgfqpoint{-0.020833in}{0.005525in}}{\pgfqpoint{-0.020833in}{0.000000in}}%
\pgfpathcurveto{\pgfqpoint{-0.020833in}{-0.005525in}}{\pgfqpoint{-0.018638in}{-0.010825in}}{\pgfqpoint{-0.014731in}{-0.014731in}}%
\pgfpathcurveto{\pgfqpoint{-0.010825in}{-0.018638in}}{\pgfqpoint{-0.005525in}{-0.020833in}}{\pgfqpoint{0.000000in}{-0.020833in}}%
\pgfpathlineto{\pgfqpoint{0.000000in}{-0.020833in}}%
\pgfpathclose%
\pgfusepath{stroke,fill}%
}%
\begin{pgfscope}%
\pgfsys@transformshift{2.633304in}{2.312500in}%
\pgfsys@useobject{currentmarker}{}%
\end{pgfscope}%
\begin{pgfscope}%
\pgfsys@transformshift{2.603328in}{2.187500in}%
\pgfsys@useobject{currentmarker}{}%
\end{pgfscope}%
\begin{pgfscope}%
\pgfsys@transformshift{2.554496in}{2.062500in}%
\pgfsys@useobject{currentmarker}{}%
\end{pgfscope}%
\begin{pgfscope}%
\pgfsys@transformshift{2.480282in}{1.812500in}%
\pgfsys@useobject{currentmarker}{}%
\end{pgfscope}%
\begin{pgfscope}%
\pgfsys@transformshift{2.367147in}{1.562500in}%
\pgfsys@useobject{currentmarker}{}%
\end{pgfscope}%
\begin{pgfscope}%
\pgfsys@transformshift{2.209049in}{1.312500in}%
\pgfsys@useobject{currentmarker}{}%
\end{pgfscope}%
\begin{pgfscope}%
\pgfsys@transformshift{2.124923in}{1.187500in}%
\pgfsys@useobject{currentmarker}{}%
\end{pgfscope}%
\begin{pgfscope}%
\pgfsys@transformshift{2.028227in}{1.062500in}%
\pgfsys@useobject{currentmarker}{}%
\end{pgfscope}%
\begin{pgfscope}%
\pgfsys@transformshift{1.930564in}{0.937500in}%
\pgfsys@useobject{currentmarker}{}%
\end{pgfscope}%
\begin{pgfscope}%
\pgfsys@transformshift{1.827099in}{0.812500in}%
\pgfsys@useobject{currentmarker}{}%
\end{pgfscope}%
\end{pgfscope}%
\begin{pgfscope}%
\pgfpathrectangle{\pgfqpoint{0.525000in}{0.450000in}}{\pgfqpoint{2.310000in}{2.475000in}}%
\pgfusepath{clip}%
\pgfsetbuttcap%
\pgfsetbeveljoin%
\definecolor{currentfill}{rgb}{0.000000,0.000000,1.000000}%
\pgfsetfillcolor{currentfill}%
\pgfsetlinewidth{1.003750pt}%
\definecolor{currentstroke}{rgb}{0.000000,0.000000,1.000000}%
\pgfsetstrokecolor{currentstroke}%
\pgfsetdash{}{0pt}%
\pgfsys@defobject{currentmarker}{\pgfqpoint{-0.039627in}{-0.033709in}}{\pgfqpoint{0.039627in}{0.041667in}}{%
\pgfpathmoveto{\pgfqpoint{0.000000in}{0.041667in}}%
\pgfpathlineto{\pgfqpoint{-0.009355in}{0.012876in}}%
\pgfpathlineto{\pgfqpoint{-0.039627in}{0.012876in}}%
\pgfpathlineto{\pgfqpoint{-0.015136in}{-0.004918in}}%
\pgfpathlineto{\pgfqpoint{-0.024491in}{-0.033709in}}%
\pgfpathlineto{\pgfqpoint{-0.000000in}{-0.015915in}}%
\pgfpathlineto{\pgfqpoint{0.024491in}{-0.033709in}}%
\pgfpathlineto{\pgfqpoint{0.015136in}{-0.004918in}}%
\pgfpathlineto{\pgfqpoint{0.039627in}{0.012876in}}%
\pgfpathlineto{\pgfqpoint{0.009355in}{0.012876in}}%
\pgfpathlineto{\pgfqpoint{0.000000in}{0.041667in}}%
\pgfpathclose%
\pgfusepath{stroke,fill}%
}%
\begin{pgfscope}%
\pgfsys@transformshift{2.730000in}{2.812500in}%
\pgfsys@useobject{currentmarker}{}%
\end{pgfscope}%
\begin{pgfscope}%
\pgfsys@transformshift{2.729758in}{2.687500in}%
\pgfsys@useobject{currentmarker}{}%
\end{pgfscope}%
\begin{pgfscope}%
\pgfsys@transformshift{2.705342in}{2.562500in}%
\pgfsys@useobject{currentmarker}{}%
\end{pgfscope}%
\begin{pgfscope}%
\pgfsys@transformshift{2.678509in}{2.437500in}%
\pgfsys@useobject{currentmarker}{}%
\end{pgfscope}%
\begin{pgfscope}%
\pgfsys@transformshift{2.529114in}{1.937500in}%
\pgfsys@useobject{currentmarker}{}%
\end{pgfscope}%
\begin{pgfscope}%
\pgfsys@transformshift{2.431450in}{1.687500in}%
\pgfsys@useobject{currentmarker}{}%
\end{pgfscope}%
\begin{pgfscope}%
\pgfsys@transformshift{2.306712in}{1.437500in}%
\pgfsys@useobject{currentmarker}{}%
\end{pgfscope}%
\begin{pgfscope}%
\pgfsys@transformshift{1.693658in}{0.687500in}%
\pgfsys@useobject{currentmarker}{}%
\end{pgfscope}%
\begin{pgfscope}%
\pgfsys@transformshift{0.630000in}{0.562500in}%
\pgfsys@useobject{currentmarker}{}%
\end{pgfscope}%
\end{pgfscope}%
\begin{pgfscope}%
\pgfsetrectcap%
\pgfsetmiterjoin%
\pgfsetlinewidth{0.803000pt}%
\definecolor{currentstroke}{rgb}{0.000000,0.000000,0.000000}%
\pgfsetstrokecolor{currentstroke}%
\pgfsetdash{}{0pt}%
\pgfpathmoveto{\pgfqpoint{0.525000in}{0.450000in}}%
\pgfpathlineto{\pgfqpoint{0.525000in}{2.925000in}}%
\pgfusepath{stroke}%
\end{pgfscope}%
\begin{pgfscope}%
\pgfsetrectcap%
\pgfsetmiterjoin%
\pgfsetlinewidth{0.803000pt}%
\definecolor{currentstroke}{rgb}{0.000000,0.000000,0.000000}%
\pgfsetstrokecolor{currentstroke}%
\pgfsetdash{}{0pt}%
\pgfpathmoveto{\pgfqpoint{2.835000in}{0.450000in}}%
\pgfpathlineto{\pgfqpoint{2.835000in}{2.925000in}}%
\pgfusepath{stroke}%
\end{pgfscope}%
\begin{pgfscope}%
\pgfsetrectcap%
\pgfsetmiterjoin%
\pgfsetlinewidth{0.803000pt}%
\definecolor{currentstroke}{rgb}{0.000000,0.000000,0.000000}%
\pgfsetstrokecolor{currentstroke}%
\pgfsetdash{}{0pt}%
\pgfpathmoveto{\pgfqpoint{0.525000in}{0.450000in}}%
\pgfpathlineto{\pgfqpoint{2.835000in}{0.450000in}}%
\pgfusepath{stroke}%
\end{pgfscope}%
\begin{pgfscope}%
\pgfsetrectcap%
\pgfsetmiterjoin%
\pgfsetlinewidth{0.803000pt}%
\definecolor{currentstroke}{rgb}{0.000000,0.000000,0.000000}%
\pgfsetstrokecolor{currentstroke}%
\pgfsetdash{}{0pt}%
\pgfpathmoveto{\pgfqpoint{0.525000in}{2.925000in}}%
\pgfpathlineto{\pgfqpoint{2.835000in}{2.925000in}}%
\pgfusepath{stroke}%
\end{pgfscope}%
\begin{pgfscope}%
\pgfpathrectangle{\pgfqpoint{2.835000in}{0.450000in}}{\pgfqpoint{0.577500in}{2.475000in}}%
\pgfusepath{clip}%
\pgfsetbuttcap%
\pgfsetmiterjoin%
\definecolor{currentfill}{rgb}{0.501961,0.501961,0.501961}%
\pgfsetfillcolor{currentfill}%
\pgfsetfillopacity{0.100000}%
\pgfsetlinewidth{0.000000pt}%
\definecolor{currentstroke}{rgb}{0.000000,0.000000,0.000000}%
\pgfsetstrokecolor{currentstroke}%
\pgfsetstrokeopacity{0.100000}%
\pgfsetdash{}{0pt}%
\pgfpathmoveto{\pgfqpoint{3.123750in}{2.545238in}}%
\pgfpathlineto{\pgfqpoint{3.412500in}{2.545238in}}%
\pgfpathlineto{\pgfqpoint{3.412500in}{2.807143in}}%
\pgfpathlineto{\pgfqpoint{3.123750in}{2.807143in}}%
\pgfpathlineto{\pgfqpoint{3.123750in}{2.545238in}}%
\pgfpathclose%
\pgfusepath{fill}%
\end{pgfscope}%
\begin{pgfscope}%
\pgfpathrectangle{\pgfqpoint{2.835000in}{0.450000in}}{\pgfqpoint{0.577500in}{2.475000in}}%
\pgfusepath{clip}%
\pgfsetbuttcap%
\pgfsetmiterjoin%
\definecolor{currentfill}{rgb}{0.501961,0.501961,0.501961}%
\pgfsetfillcolor{currentfill}%
\pgfsetfillopacity{0.150000}%
\pgfsetlinewidth{0.000000pt}%
\definecolor{currentstroke}{rgb}{0.000000,0.000000,0.000000}%
\pgfsetstrokecolor{currentstroke}%
\pgfsetstrokeopacity{0.150000}%
\pgfsetdash{}{0pt}%
\pgfpathmoveto{\pgfqpoint{3.123750in}{2.283333in}}%
\pgfpathlineto{\pgfqpoint{3.412500in}{2.283333in}}%
\pgfpathlineto{\pgfqpoint{3.412500in}{2.545238in}}%
\pgfpathlineto{\pgfqpoint{3.123750in}{2.545238in}}%
\pgfpathlineto{\pgfqpoint{3.123750in}{2.283333in}}%
\pgfpathclose%
\pgfusepath{fill}%
\end{pgfscope}%
\begin{pgfscope}%
\pgfpathrectangle{\pgfqpoint{2.835000in}{0.450000in}}{\pgfqpoint{0.577500in}{2.475000in}}%
\pgfusepath{clip}%
\pgfsetbuttcap%
\pgfsetmiterjoin%
\definecolor{currentfill}{rgb}{0.501961,0.501961,0.501961}%
\pgfsetfillcolor{currentfill}%
\pgfsetfillopacity{0.200000}%
\pgfsetlinewidth{0.000000pt}%
\definecolor{currentstroke}{rgb}{0.000000,0.000000,0.000000}%
\pgfsetstrokecolor{currentstroke}%
\pgfsetstrokeopacity{0.200000}%
\pgfsetdash{}{0pt}%
\pgfpathmoveto{\pgfqpoint{3.123750in}{2.021429in}}%
\pgfpathlineto{\pgfqpoint{3.412500in}{2.021429in}}%
\pgfpathlineto{\pgfqpoint{3.412500in}{2.283333in}}%
\pgfpathlineto{\pgfqpoint{3.123750in}{2.283333in}}%
\pgfpathlineto{\pgfqpoint{3.123750in}{2.021429in}}%
\pgfpathclose%
\pgfusepath{fill}%
\end{pgfscope}%
\begin{pgfscope}%
\pgfpathrectangle{\pgfqpoint{2.835000in}{0.450000in}}{\pgfqpoint{0.577500in}{2.475000in}}%
\pgfusepath{clip}%
\pgfsetbuttcap%
\pgfsetmiterjoin%
\definecolor{currentfill}{rgb}{0.501961,0.501961,0.501961}%
\pgfsetfillcolor{currentfill}%
\pgfsetfillopacity{0.250000}%
\pgfsetlinewidth{0.000000pt}%
\definecolor{currentstroke}{rgb}{0.000000,0.000000,0.000000}%
\pgfsetstrokecolor{currentstroke}%
\pgfsetstrokeopacity{0.250000}%
\pgfsetdash{}{0pt}%
\pgfpathmoveto{\pgfqpoint{3.123750in}{1.759524in}}%
\pgfpathlineto{\pgfqpoint{3.412500in}{1.759524in}}%
\pgfpathlineto{\pgfqpoint{3.412500in}{2.021429in}}%
\pgfpathlineto{\pgfqpoint{3.123750in}{2.021429in}}%
\pgfpathlineto{\pgfqpoint{3.123750in}{1.759524in}}%
\pgfpathclose%
\pgfusepath{fill}%
\end{pgfscope}%
\begin{pgfscope}%
\pgfpathrectangle{\pgfqpoint{2.835000in}{0.450000in}}{\pgfqpoint{0.577500in}{2.475000in}}%
\pgfusepath{clip}%
\pgfsetbuttcap%
\pgfsetmiterjoin%
\definecolor{currentfill}{rgb}{0.501961,0.501961,0.501961}%
\pgfsetfillcolor{currentfill}%
\pgfsetfillopacity{0.300000}%
\pgfsetlinewidth{0.000000pt}%
\definecolor{currentstroke}{rgb}{0.000000,0.000000,0.000000}%
\pgfsetstrokecolor{currentstroke}%
\pgfsetstrokeopacity{0.300000}%
\pgfsetdash{}{0pt}%
\pgfpathmoveto{\pgfqpoint{3.123750in}{1.497619in}}%
\pgfpathlineto{\pgfqpoint{3.412500in}{1.497619in}}%
\pgfpathlineto{\pgfqpoint{3.412500in}{1.759524in}}%
\pgfpathlineto{\pgfqpoint{3.123750in}{1.759524in}}%
\pgfpathlineto{\pgfqpoint{3.123750in}{1.497619in}}%
\pgfpathclose%
\pgfusepath{fill}%
\end{pgfscope}%
\begin{pgfscope}%
\pgfpathrectangle{\pgfqpoint{2.835000in}{0.450000in}}{\pgfqpoint{0.577500in}{2.475000in}}%
\pgfusepath{clip}%
\pgfsetbuttcap%
\pgfsetmiterjoin%
\definecolor{currentfill}{rgb}{0.501961,0.501961,0.501961}%
\pgfsetfillcolor{currentfill}%
\pgfsetfillopacity{0.350000}%
\pgfsetlinewidth{0.000000pt}%
\definecolor{currentstroke}{rgb}{0.000000,0.000000,0.000000}%
\pgfsetstrokecolor{currentstroke}%
\pgfsetstrokeopacity{0.350000}%
\pgfsetdash{}{0pt}%
\pgfpathmoveto{\pgfqpoint{3.123750in}{1.235714in}}%
\pgfpathlineto{\pgfqpoint{3.412500in}{1.235714in}}%
\pgfpathlineto{\pgfqpoint{3.412500in}{1.497619in}}%
\pgfpathlineto{\pgfqpoint{3.123750in}{1.497619in}}%
\pgfpathlineto{\pgfqpoint{3.123750in}{1.235714in}}%
\pgfpathclose%
\pgfusepath{fill}%
\end{pgfscope}%
\begin{pgfscope}%
\pgfpathrectangle{\pgfqpoint{2.835000in}{0.450000in}}{\pgfqpoint{0.577500in}{2.475000in}}%
\pgfusepath{clip}%
\pgfsetbuttcap%
\pgfsetmiterjoin%
\definecolor{currentfill}{rgb}{0.501961,0.501961,0.501961}%
\pgfsetfillcolor{currentfill}%
\pgfsetfillopacity{0.400000}%
\pgfsetlinewidth{0.000000pt}%
\definecolor{currentstroke}{rgb}{0.000000,0.000000,0.000000}%
\pgfsetstrokecolor{currentstroke}%
\pgfsetstrokeopacity{0.400000}%
\pgfsetdash{}{0pt}%
\pgfpathmoveto{\pgfqpoint{3.123750in}{0.973810in}}%
\pgfpathlineto{\pgfqpoint{3.412500in}{0.973810in}}%
\pgfpathlineto{\pgfqpoint{3.412500in}{1.235714in}}%
\pgfpathlineto{\pgfqpoint{3.123750in}{1.235714in}}%
\pgfpathlineto{\pgfqpoint{3.123750in}{0.973810in}}%
\pgfpathclose%
\pgfusepath{fill}%
\end{pgfscope}%
\begin{pgfscope}%
\pgfpathrectangle{\pgfqpoint{2.835000in}{0.450000in}}{\pgfqpoint{0.577500in}{2.475000in}}%
\pgfusepath{clip}%
\pgfsetbuttcap%
\pgfsetmiterjoin%
\definecolor{currentfill}{rgb}{0.501961,0.501961,0.501961}%
\pgfsetfillcolor{currentfill}%
\pgfsetfillopacity{0.450000}%
\pgfsetlinewidth{0.000000pt}%
\definecolor{currentstroke}{rgb}{0.000000,0.000000,0.000000}%
\pgfsetstrokecolor{currentstroke}%
\pgfsetstrokeopacity{0.450000}%
\pgfsetdash{}{0pt}%
\pgfpathmoveto{\pgfqpoint{3.123750in}{0.711905in}}%
\pgfpathlineto{\pgfqpoint{3.412500in}{0.711905in}}%
\pgfpathlineto{\pgfqpoint{3.412500in}{0.973810in}}%
\pgfpathlineto{\pgfqpoint{3.123750in}{0.973810in}}%
\pgfpathlineto{\pgfqpoint{3.123750in}{0.711905in}}%
\pgfpathclose%
\pgfusepath{fill}%
\end{pgfscope}%
\begin{pgfscope}%
\pgfpathrectangle{\pgfqpoint{2.835000in}{0.450000in}}{\pgfqpoint{0.577500in}{2.475000in}}%
\pgfusepath{clip}%
\pgfsetbuttcap%
\pgfsetmiterjoin%
\definecolor{currentfill}{rgb}{0.501961,0.501961,0.501961}%
\pgfsetfillcolor{currentfill}%
\pgfsetfillopacity{0.500000}%
\pgfsetlinewidth{0.000000pt}%
\definecolor{currentstroke}{rgb}{0.000000,0.000000,0.000000}%
\pgfsetstrokecolor{currentstroke}%
\pgfsetstrokeopacity{0.500000}%
\pgfsetdash{}{0pt}%
\pgfpathmoveto{\pgfqpoint{3.123750in}{0.450000in}}%
\pgfpathlineto{\pgfqpoint{3.412500in}{0.450000in}}%
\pgfpathlineto{\pgfqpoint{3.412500in}{0.711905in}}%
\pgfpathlineto{\pgfqpoint{3.123750in}{0.711905in}}%
\pgfpathlineto{\pgfqpoint{3.123750in}{0.450000in}}%
\pgfpathclose%
\pgfusepath{fill}%
\end{pgfscope}%
\begin{pgfscope}%
\pgfsetroundcap%
\pgfsetroundjoin%
\pgfsetlinewidth{0.501875pt}%
\definecolor{currentstroke}{rgb}{0.000000,0.000000,0.000000}%
\pgfsetstrokecolor{currentstroke}%
\pgfsetdash{}{0pt}%
\pgfpathmoveto{\pgfqpoint{3.117975in}{2.676190in}}%
\pgfpathquadraticcurveto{\pgfqpoint{2.926163in}{2.744345in}}{\pgfqpoint{2.741667in}{2.809900in}}%
\pgfusepath{stroke}%
\end{pgfscope}%
\begin{pgfscope}%
\pgfsetroundcap%
\pgfsetroundjoin%
\pgfsetlinewidth{0.501875pt}%
\definecolor{currentstroke}{rgb}{0.000000,0.000000,0.000000}%
\pgfsetstrokecolor{currentstroke}%
\pgfsetdash{}{0pt}%
\pgfpathmoveto{\pgfqpoint{2.784716in}{2.765125in}}%
\pgfpathlineto{\pgfqpoint{2.741667in}{2.809900in}}%
\pgfpathlineto{\pgfqpoint{2.803317in}{2.817474in}}%
\pgfusepath{stroke}%
\end{pgfscope}%
\begin{pgfscope}%
\pgfsetroundcap%
\pgfsetroundjoin%
\pgfsetlinewidth{0.501875pt}%
\definecolor{currentstroke}{rgb}{0.000000,0.000000,0.000000}%
\pgfsetstrokecolor{currentstroke}%
\pgfsetdash{}{0pt}%
\pgfpathmoveto{\pgfqpoint{3.117975in}{2.414286in}}%
\pgfpathquadraticcurveto{\pgfqpoint{2.926042in}{2.550893in}}{\pgfqpoint{2.740435in}{2.682998in}}%
\pgfusepath{stroke}%
\end{pgfscope}%
\begin{pgfscope}%
\pgfsetroundcap%
\pgfsetroundjoin%
\pgfsetlinewidth{0.501875pt}%
\definecolor{currentstroke}{rgb}{0.000000,0.000000,0.000000}%
\pgfsetstrokecolor{currentstroke}%
\pgfsetdash{}{0pt}%
\pgfpathmoveto{\pgfqpoint{2.769589in}{2.628152in}}%
\pgfpathlineto{\pgfqpoint{2.740435in}{2.682998in}}%
\pgfpathlineto{\pgfqpoint{2.801804in}{2.673414in}}%
\pgfusepath{stroke}%
\end{pgfscope}%
\begin{pgfscope}%
\pgfsetroundcap%
\pgfsetroundjoin%
\pgfsetlinewidth{0.501875pt}%
\definecolor{currentstroke}{rgb}{0.000000,0.000000,0.000000}%
\pgfsetstrokecolor{currentstroke}%
\pgfsetdash{}{0pt}%
\pgfpathmoveto{\pgfqpoint{3.117975in}{2.152381in}}%
\pgfpathquadraticcurveto{\pgfqpoint{2.913834in}{2.357440in}}{\pgfqpoint{2.715172in}{2.556998in}}%
\pgfusepath{stroke}%
\end{pgfscope}%
\begin{pgfscope}%
\pgfsetroundcap%
\pgfsetroundjoin%
\pgfsetlinewidth{0.501875pt}%
\definecolor{currentstroke}{rgb}{0.000000,0.000000,0.000000}%
\pgfsetstrokecolor{currentstroke}%
\pgfsetdash{}{0pt}%
\pgfpathmoveto{\pgfqpoint{2.734681in}{2.498028in}}%
\pgfpathlineto{\pgfqpoint{2.715172in}{2.556998in}}%
\pgfpathlineto{\pgfqpoint{2.774053in}{2.537223in}}%
\pgfusepath{stroke}%
\end{pgfscope}%
\begin{pgfscope}%
\pgfsetroundcap%
\pgfsetroundjoin%
\pgfsetlinewidth{0.501875pt}%
\definecolor{currentstroke}{rgb}{0.000000,0.000000,0.000000}%
\pgfsetstrokecolor{currentstroke}%
\pgfsetdash{}{0pt}%
\pgfpathmoveto{\pgfqpoint{3.117975in}{1.890476in}}%
\pgfpathquadraticcurveto{\pgfqpoint{2.900418in}{2.163988in}}{\pgfqpoint{2.687694in}{2.431424in}}%
\pgfusepath{stroke}%
\end{pgfscope}%
\begin{pgfscope}%
\pgfsetroundcap%
\pgfsetroundjoin%
\pgfsetlinewidth{0.501875pt}%
\definecolor{currentstroke}{rgb}{0.000000,0.000000,0.000000}%
\pgfsetstrokecolor{currentstroke}%
\pgfsetdash{}{0pt}%
\pgfpathmoveto{\pgfqpoint{2.700538in}{2.370653in}}%
\pgfpathlineto{\pgfqpoint{2.687694in}{2.431424in}}%
\pgfpathlineto{\pgfqpoint{2.744017in}{2.405237in}}%
\pgfusepath{stroke}%
\end{pgfscope}%
\begin{pgfscope}%
\pgfsetroundcap%
\pgfsetroundjoin%
\pgfsetlinewidth{0.501875pt}%
\definecolor{currentstroke}{rgb}{0.000000,0.000000,0.000000}%
\pgfsetstrokecolor{currentstroke}%
\pgfsetdash{}{0pt}%
\pgfpathmoveto{\pgfqpoint{3.117975in}{1.628571in}}%
\pgfpathquadraticcurveto{\pgfqpoint{2.825720in}{1.783036in}}{\pgfqpoint{2.540329in}{1.933872in}}%
\pgfusepath{stroke}%
\end{pgfscope}%
\begin{pgfscope}%
\pgfsetroundcap%
\pgfsetroundjoin%
\pgfsetlinewidth{0.501875pt}%
\definecolor{currentstroke}{rgb}{0.000000,0.000000,0.000000}%
\pgfsetstrokecolor{currentstroke}%
\pgfsetdash{}{0pt}%
\pgfpathmoveto{\pgfqpoint{2.576467in}{1.883354in}}%
\pgfpathlineto{\pgfqpoint{2.540329in}{1.933872in}}%
\pgfpathlineto{\pgfqpoint{2.602426in}{1.932471in}}%
\pgfusepath{stroke}%
\end{pgfscope}%
\begin{pgfscope}%
\pgfsetroundcap%
\pgfsetroundjoin%
\pgfsetlinewidth{0.501875pt}%
\definecolor{currentstroke}{rgb}{0.000000,0.000000,0.000000}%
\pgfsetstrokecolor{currentstroke}%
\pgfsetdash{}{0pt}%
\pgfpathmoveto{\pgfqpoint{3.117975in}{1.366667in}}%
\pgfpathquadraticcurveto{\pgfqpoint{2.776888in}{1.527083in}}{\pgfqpoint{2.442828in}{1.684196in}}%
\pgfusepath{stroke}%
\end{pgfscope}%
\begin{pgfscope}%
\pgfsetroundcap%
\pgfsetroundjoin%
\pgfsetlinewidth{0.501875pt}%
\definecolor{currentstroke}{rgb}{0.000000,0.000000,0.000000}%
\pgfsetstrokecolor{currentstroke}%
\pgfsetdash{}{0pt}%
\pgfpathmoveto{\pgfqpoint{2.481279in}{1.635415in}}%
\pgfpathlineto{\pgfqpoint{2.442828in}{1.684196in}}%
\pgfpathlineto{\pgfqpoint{2.504923in}{1.685688in}}%
\pgfusepath{stroke}%
\end{pgfscope}%
\begin{pgfscope}%
\pgfsetroundcap%
\pgfsetroundjoin%
\pgfsetlinewidth{0.501875pt}%
\definecolor{currentstroke}{rgb}{0.000000,0.000000,0.000000}%
\pgfsetstrokecolor{currentstroke}%
\pgfsetdash{}{0pt}%
\pgfpathmoveto{\pgfqpoint{3.117975in}{1.104762in}}%
\pgfpathquadraticcurveto{\pgfqpoint{2.714519in}{1.271131in}}{\pgfqpoint{2.318241in}{1.434540in}}%
\pgfusepath{stroke}%
\end{pgfscope}%
\begin{pgfscope}%
\pgfsetroundcap%
\pgfsetroundjoin%
\pgfsetlinewidth{0.501875pt}%
\definecolor{currentstroke}{rgb}{0.000000,0.000000,0.000000}%
\pgfsetstrokecolor{currentstroke}%
\pgfsetdash{}{0pt}%
\pgfpathmoveto{\pgfqpoint{2.359012in}{1.387681in}}%
\pgfpathlineto{\pgfqpoint{2.318241in}{1.434540in}}%
\pgfpathlineto{\pgfqpoint{2.380191in}{1.439041in}}%
\pgfusepath{stroke}%
\end{pgfscope}%
\begin{pgfscope}%
\pgfsetroundcap%
\pgfsetroundjoin%
\pgfsetlinewidth{0.501875pt}%
\definecolor{currentstroke}{rgb}{0.000000,0.000000,0.000000}%
\pgfsetstrokecolor{currentstroke}%
\pgfsetdash{}{0pt}%
\pgfpathmoveto{\pgfqpoint{3.117975in}{0.842857in}}%
\pgfpathquadraticcurveto{\pgfqpoint{2.407992in}{0.765179in}}{\pgfqpoint{1.705728in}{0.688344in}}%
\pgfusepath{stroke}%
\end{pgfscope}%
\begin{pgfscope}%
\pgfsetroundcap%
\pgfsetroundjoin%
\pgfsetlinewidth{0.501875pt}%
\definecolor{currentstroke}{rgb}{0.000000,0.000000,0.000000}%
\pgfsetstrokecolor{currentstroke}%
\pgfsetdash{}{0pt}%
\pgfpathmoveto{\pgfqpoint{1.763975in}{0.666774in}}%
\pgfpathlineto{\pgfqpoint{1.705728in}{0.688344in}}%
\pgfpathlineto{\pgfqpoint{1.757933in}{0.722000in}}%
\pgfusepath{stroke}%
\end{pgfscope}%
\begin{pgfscope}%
\pgfsetroundcap%
\pgfsetroundjoin%
\pgfsetlinewidth{0.501875pt}%
\definecolor{currentstroke}{rgb}{0.000000,0.000000,0.000000}%
\pgfsetstrokecolor{currentstroke}%
\pgfsetdash{}{0pt}%
\pgfpathmoveto{\pgfqpoint{3.117975in}{0.580952in}}%
\pgfpathquadraticcurveto{\pgfqpoint{1.876163in}{0.571726in}}{\pgfqpoint{0.642115in}{0.562558in}}%
\pgfusepath{stroke}%
\end{pgfscope}%
\begin{pgfscope}%
\pgfsetroundcap%
\pgfsetroundjoin%
\pgfsetlinewidth{0.501875pt}%
\definecolor{currentstroke}{rgb}{0.000000,0.000000,0.000000}%
\pgfsetstrokecolor{currentstroke}%
\pgfsetdash{}{0pt}%
\pgfpathmoveto{\pgfqpoint{0.697876in}{0.535193in}}%
\pgfpathlineto{\pgfqpoint{0.642115in}{0.562558in}}%
\pgfpathlineto{\pgfqpoint{0.697463in}{0.590747in}}%
\pgfusepath{stroke}%
\end{pgfscope}%
\begin{pgfscope}%
\definecolor{textcolor}{rgb}{0.000000,0.000000,0.000000}%
\pgfsetstrokecolor{textcolor}%
\pgfsetfillcolor{textcolor}%
\pgftext[x=3.268125in,y=2.676190in,,]{\color{textcolor}\sffamily\fontsize{10.000000}{12.000000}\selectfont \(\displaystyle \alpha_{1}\)}%
\end{pgfscope}%
\begin{pgfscope}%
\definecolor{textcolor}{rgb}{0.000000,0.000000,0.000000}%
\pgfsetstrokecolor{textcolor}%
\pgfsetfillcolor{textcolor}%
\pgftext[x=3.268125in,y=2.414286in,,]{\color{textcolor}\sffamily\fontsize{10.000000}{12.000000}\selectfont \(\displaystyle \alpha_{2}\)}%
\end{pgfscope}%
\begin{pgfscope}%
\definecolor{textcolor}{rgb}{0.000000,0.000000,0.000000}%
\pgfsetstrokecolor{textcolor}%
\pgfsetfillcolor{textcolor}%
\pgftext[x=3.268125in,y=2.152381in,,]{\color{textcolor}\sffamily\fontsize{10.000000}{12.000000}\selectfont \(\displaystyle \alpha_{3}\)}%
\end{pgfscope}%
\begin{pgfscope}%
\definecolor{textcolor}{rgb}{0.000000,0.000000,0.000000}%
\pgfsetstrokecolor{textcolor}%
\pgfsetfillcolor{textcolor}%
\pgftext[x=3.268125in,y=1.890476in,,]{\color{textcolor}\sffamily\fontsize{10.000000}{12.000000}\selectfont \(\displaystyle \alpha_{4}\)}%
\end{pgfscope}%
\begin{pgfscope}%
\definecolor{textcolor}{rgb}{0.000000,0.000000,0.000000}%
\pgfsetstrokecolor{textcolor}%
\pgfsetfillcolor{textcolor}%
\pgftext[x=3.268125in,y=1.628571in,,]{\color{textcolor}\sffamily\fontsize{10.000000}{12.000000}\selectfont \(\displaystyle \alpha_{8}\)}%
\end{pgfscope}%
\begin{pgfscope}%
\definecolor{textcolor}{rgb}{0.000000,0.000000,0.000000}%
\pgfsetstrokecolor{textcolor}%
\pgfsetfillcolor{textcolor}%
\pgftext[x=3.268125in,y=1.366667in,,]{\color{textcolor}\sffamily\fontsize{10.000000}{12.000000}\selectfont \(\displaystyle \alpha_{10}\)}%
\end{pgfscope}%
\begin{pgfscope}%
\definecolor{textcolor}{rgb}{0.000000,0.000000,0.000000}%
\pgfsetstrokecolor{textcolor}%
\pgfsetfillcolor{textcolor}%
\pgftext[x=3.268125in,y=1.104762in,,]{\color{textcolor}\sffamily\fontsize{10.000000}{12.000000}\selectfont \(\displaystyle \alpha_{12}\)}%
\end{pgfscope}%
\begin{pgfscope}%
\definecolor{textcolor}{rgb}{0.000000,0.000000,0.000000}%
\pgfsetstrokecolor{textcolor}%
\pgfsetfillcolor{textcolor}%
\pgftext[x=3.268125in,y=0.842857in,,]{\color{textcolor}\sffamily\fontsize{10.000000}{12.000000}\selectfont \(\displaystyle \alpha_{18}\)}%
\end{pgfscope}%
\begin{pgfscope}%
\definecolor{textcolor}{rgb}{0.000000,0.000000,0.000000}%
\pgfsetstrokecolor{textcolor}%
\pgfsetfillcolor{textcolor}%
\pgftext[x=3.268125in,y=0.580952in,,]{\color{textcolor}\sffamily\fontsize{10.000000}{12.000000}\selectfont \(\displaystyle \alpha_{19}\)}%
\end{pgfscope}%
\end{pgfpicture}%
\makeatother%
\endgroup%

%% file: tables/tab1_cheb_fixed_delta.tex
\begin{table*}[t]
    
    \tiny
    \setlength{\tabcolsep}{2pt}
    \centering
    \caption{Results for same dispersion using Chebyshev's inequality}
    \label{tab:cheb-same-delta}
    \begin{tabular}{l|l|r|l|r r c|r r c|r r c|r r c|r r c}
    \toprule
       ~ & \multicolumn{1}{c|}{\multirow{2}{*}{B}}  &\multicolumn{1}{c|}{\multirow{2}{*}{$\alpha$}} & \multicolumn{1}{c|}{\multirow{2}{*}{$\delta$}} &
        \multicolumn{3}{c|}{GSEMO with g(x) (1)} & \multicolumn{3}{c|}{GSEMO+Filtering with g(x) (2)}  & 
        \multicolumn{3}{c|}{NSGA-II with g(x) (3)}& 
        \multicolumn{3}{c|}{NSGA-II with g$^\prime$(x) (4)}& 
        \multicolumn{3}{c}{NSGA-II with g$^{\prime\prime}$(x) (5)}\\
        ~ & ~ & ~ & ~ & \multicolumn{1}{c}{mean} & \multicolumn{1}{c}{std} & \multicolumn{1}{c|}{stat} & 
        \multicolumn{1}{c}{mean} & \multicolumn{1}{c}{std} & \multicolumn{1}{c|}{stat} & 
        \multicolumn{1}{c}{mean} & \multicolumn{1}{c}{std} & \multicolumn{1}{c|}{stat} & 
        \multicolumn{1}{c}{mean} & \multicolumn{1}{c}{std} & \multicolumn{1}{c|}{stat} & 
        \multicolumn{1}{c}{mean} & \multicolumn{1}{c}{std} & \multicolumn{1}{c}{stat} \\ \hline
        
\multirow{6}{*}{\rotatebox[origin=c]{90}{uncorr-100}}  & \multirow{6}{*}{\rotatebox[origin=c]{90}{2407}} & \multirow{2}{*}{0.1} & 25 &	 11029.6826 &	 76.3191 &	 $2^{-}   5^{+}$ &	\textbf{11085.6072} &	 0.0000 &	 $1^{+}  3^{+} 4^{+} 5^{+}$ &	 11007.6251 &	 88.8419 &	 $2^{-}$ &	 10998.7574 &	 50.6310 &	 $2^{-}$ &	 11007.7624 &	 52.5740 &	 $1^{-} 2^{-}$ \\	
&  &  & 50 &	 10862.4750 &	 61.4080 &	 $2^{-}   5^{+}$ &	\textbf{10907.0000} &	 0.0000 &	 $1^{+}  3^{+} 4^{+} 5^{+}$ &	 10837.2841 &	 53.1386 &	 $2^{-}$ &	 10811.4832 &	 80.3604 &	 $2^{-}$ &	 10832.9581 &	 50.1309 &	 $1^{-} 2^{-}$ \\	
&  & \multirow{2}{*}{0.01}  & 25 &	 10620.5264 &	 71.0446 &	 $2^{-}   5^{+}$ &	\textbf{10672.3379} &	 0.0000 &	 $1^{+}  3^{+} 4^{+} 5^{+}$ &	 10602.1969 &	 81.4896 &	 $2^{-}$ &	 10594.9482 &	 45.5447 &	 $2^{-}$ &	 10602.8065 &	 46.9209 &	 $1^{-} 2^{-}$ \\	
&  &  & 50 &	 10044.4941 &	 49.7668 &	 $2^{-}   5^{+}$ &	\textbf{10079.9649} &	 0.0000 &	 $1^{+}  3^{+} 4^{+} 5^{+}$ &	 10025.8881 &	 41.6418 &	 $2^{-}$ &	 10004.2550 &	 67.4364 &	 $2^{-}$ &	 10023.0463 &	 38.8802 &	 $1^{-} 2^{-}$ \\	
&  & \multirow{2}{*}{0.001}  &  25 &	 9345.9245 &	 55.2560 &	 $2^{-}   5^{+}$ &	\textbf{9384.5150} &	 0.0000 &	 $1^{+}  3^{+} 4^{+} 5^{+}$ &	 9338.8078 &	 59.0630 &	 $2^{-}   5^{+}$ &	 9336.6045 &	 30.1889 &	 $2^{-}$ &	 9340.8894 &	 29.5604 &	 $1^{-} 2^{-} 3^{-}$ \\	
&  &  & 50 &	 7495.5153 &	 17.9630 &	 $2^{-}   5^{+}$ &	\textbf{7502.7716} &	 0.0000 &	 $1^{+}  3^{+} 4^{+} 5^{+}$ &	 7498.7722 &	 13.1282 &	 $2^{-}   5^{+}$ &	 7488.8368 &	 31.6085 &	 $2^{-}$ &	 7499.2121 &	 9.5267 &	 $1^{-} 2^{-} 3^{-}$ \\	\hline

\multirow{6}{*}{\rotatebox[origin=c]{90}{strong-100}}  & \multirow{6}{*}{\rotatebox[origin=c]{90}{4187}} & \multirow{2}{*}{0.1} & 25 &	 8507.0584 &	 130.1420 &	 $2^{-}$ &	\textbf{8606.3413} &	 87.2952 &	 $1^{+}  3^{+} 4^{+} 5^{+}$ &	 8525.5561 &	 125.6131 &	 $2^{-}$ &	 8500.0805 &	 149.4958 &	 $2^{-}$ &	 8499.3632 &	 124.5506 &	 $2^{-}$ \\	
&  &  & 50 &	 8368.0306 &	 94.0118 &	 $2^{-}$ &	\textbf{8422.6322} &	 67.2273 &	 $1^{+}  3^{+}  5^{+}$ &	 8326.9385 &	 114.3897 &	 $2^{-}$ &	 8364.0549 &	 124.7364 &	 ~  &	 8319.3263 &	 115.3753 &	 $2^{-}$ \\	
&  & \multirow{2}{*}{0.01}  & 25 &	 8083.9678 &	 106.1983 &	 $2^{-}$ &	\textbf{8170.3360} &	 69.5792 &	 $1^{+}   4^{+} 5^{+}$ &	 8102.4388 &	 104.2026 &	 ~  &	 8082.3479 &	 124.2002 &	 $2^{-}$ &	 8082.2853 &	 103.3765 &	 $2^{-}$ \\	
&  &  & 50 &	 7502.9361 &	 59.8544 &	 $2^{-}$ &	\textbf{7549.4937} &	 41.2749 &	 $1^{+}  3^{+}  5^{+}$ &	 7489.7223 &	 73.5765 &	 $2^{-}$ &	 7513.5174 &	 78.6964 &	 ~  &	 7485.1707 &	 73.8490 &	 $2^{-}$ \\	
&  & \multirow{2}{*}{0.001}  &  25 &	 6770.3193 &	 41.1347 &	 $2^{-}$ &	\textbf{6814.1197} &	 20.7669 &	 $1^{+}   4^{+} 5^{+}$ &	 6787.5739 &	 42.8885 &	 ~  &	 6782.2797 &	 48.8989 &	 $2^{-}$ &	 6784.1694 &	 42.0052 &	 $2^{-}$ \\	
&  &  & 50 &	 4957.5449 &	 36.7770 &	 $2^{+} 3^{-} 4^{-}$ &	 4894.0039 &	 60.2107 &	 $1^{-}  3^{-} 4^{-} 5^{-}$ &	 4990.7637 &	 17.9608 &	 $1^{+} 2^{+}   5^{+}$ &	 4989.4145 &	 13.2248 &	 $1^{+} 2^{+}   5^{+}$ &	\textbf{4991.4701} &	 11.0454 &	 $2^{+} 3^{-} 4^{-}$ \\	\hline

\multirow{6}{*}{\rotatebox[origin=c]{90}{uncorr-300}}  & \multirow{6}{*}{\rotatebox[origin=c]{90}{6853}} & \multirow{2}{*}{0.1} & 25 &	 33935.4067 &	 205.6247 &	 $2^{-}$ &	\textbf{34286.1802} &	 147.5309 &	 $1^{+}  3^{+} 4^{+} 5^{+}$ &	 33681.6883 &	 534.3217 &	 $2^{-}$ &	 33671.7620 &	 555.4450 &	 $2^{-}$ &	 33615.8492 &	 489.0223 &	 $2^{-}$ \\	
&  &  & 50 &	 33571.9980 &	 260.8593 &	 $2^{-}$ &	\textbf{33967.3813} &	 159.2433 &	 $1^{+}  3^{+} 4^{+} 5^{+}$ &	 33418.4882 &	 512.5693 &	 $2^{-}$ &	 33284.5989 &	 450.6255 &	 $2^{-}$ &	 33319.4651 &	 483.5459 &	 $2^{-}$ \\	
&  & \multirow{2}{*}{0.01}  & 25 &	 33237.8865 &	 200.4641 &	 $2^{-}$ &	\textbf{33577.9421} &	 141.7536 &	 $1^{+}  3^{+} 4^{+} 5^{+}$ &	 32992.6756 &	 521.4969 &	 $2^{-}$ &	 32984.2543 &	 541.8390 &	 $2^{-}$ &	 32929.2384 &	 476.3437 &	 $2^{-}$ \\	
&  &  & 50 &	 32180.0106 &	 245.8773 &	 $2^{-}$ &	\textbf{32551.5342} &	 144.8963 &	 $1^{+}  3^{+} 4^{+} 5^{+}$ &	 32039.6829 &	 487.7705 &	 $2^{-}$ &	 31913.5405 &	 429.4896 &	 $2^{-}$ &	 31946.2435 &	 458.2449 &	 $2^{-}$ \\	
&  & \multirow{2}{*}{0.001}  &  25 &	 31066.6084 &	 186.3571 &	 $2^{-}$ &	\textbf{31372.5619} &	 122.8614 &	 $1^{+}  3^{+} 4^{+} 5^{+}$ &	 30846.1243 &	 482.1829 &	 $2^{-}$ &	 30841.9917 &	 499.6822 &	 $2^{-}$ &	 30789.6329 &	 437.0386 &	 $2^{-}$ \\	
&  &  & 50 &	 27843.2948 &	 203.7335 &	 $2^{-}$ &	\textbf{28141.7188} &	 105.9385 &	 $1^{+}  3^{+} 4^{+} 5^{+}$ &	 27745.1612 &	 411.5637 &	 $2^{-}$ &	 27641.0702 &	 364.7302 &	 $2^{-}$ &	 27668.0494 &	 380.2576 &	 $2^{-}$ \\	\hline

\multirow{6}{*}{\rotatebox[origin=c]{90}{strong-300}}  & \multirow{6}{*}{\rotatebox[origin=c]{90}{13821}} & \multirow{2}{*}{0.1} & 25 &	 23809.6581 &	 433.2506 &	 $2^{-} 3^{-}  5^{-}$ &	\textbf{24369.6211} &	 216.9574 &	 $1^{+}  3^{+} 4^{+}$ &	 24099.5570 &	 327.0870 &	 $1^{+} 2^{-}$ &	 23986.4018 &	 344.2409 &	 $2^{-}$ &	 24176.0891 &	 232.7994 &	 $1^{+}$ \\	
&  &  & 50 &	 23594.2993 &	 335.6481 &	 $2^{-} 3^{-}  5^{-}$ &	\textbf{24135.2769} &	 220.4491 &	 $1^{+}  3^{+} 4^{+} 5^{+}$ &	 23867.9086 &	 293.6342 &	 $1^{+} 2^{-}$ &	 23695.2127 &	 304.2994 &	 $2^{-}$ &	 23899.3116 &	 223.8869 &	 $1^{+} 2^{-}$ \\	
&  & \multirow{2}{*}{0.01}  & 25 &	 23176.9548 &	 406.2664 &	 $2^{-} 3^{-}  5^{-}$ &	\textbf{23703.0401} &	 197.3436 &	 $1^{+}  3^{+} 4^{+}$ &	 23464.5667 &	 299.1984 &	 $1^{+} 2^{-}$ &	 23360.1013 &	 318.6357 &	 $2^{-}$ &	 23534.8995 &	 212.1891 &	 $1^{+}$ \\	
&  &  & 50 &	 22322.7651 &	 282.0126 &	 $2^{-} 3^{-}  5^{-}$ &	\textbf{22797.4912} &	 177.4246 &	 $1^{+}  3^{+} 4^{+} 5^{+}$ &	 22588.5433 &	 240.6771 &	 $1^{+} 2^{-}$ &	 22446.1950 &	 255.5942 &	 $2^{-}$ &	 22616.9323 &	 182.9564 &	 $1^{+} 2^{-}$ \\	
&  & \multirow{2}{*}{0.001}  &  25 &	 21208.9163 &	 322.7908 &	 $2^{-} 3^{-}  5^{-}$ &	\textbf{21626.7053} &	 138.6421 &	 $1^{+}   4^{+}$ &	 21486.9475 &	 214.2326 &	 $1^{+}$ &	 21411.6675 &	 240.6127 &	 $2^{-}$ &	 21536.8345 &	 149.2069 &	 $1^{+}$ \\	
&  &  & 50 &	 18388.5805 &	 159.4458 &	 $2^{-} 3^{-} 4^{-} 5^{-}$ &	\textbf{18647.0894} &	 70.1852 &	 $1^{+}   4^{+}$ &	 18623.6627 &	 98.1567 &	 $1^{+}$ &	 18569.9473 &	 110.1685 &	 $1^{+} 2^{-}$ &	 18638.7270 &	 63.3870 &	 $1^{+}$ \\	\hline

\multirow{6}{*}{\rotatebox[origin=c]{90}{uncorr-500}}  & \multirow{6}{*}{\rotatebox[origin=c]{90}{11243}} & \multirow{2}{*}{0.1} & 25 &	 57076.8361 &	 748.9305 &	 $2^{-} 3^{+} 4^{+}$ &	\textbf{58431.4168} &	 311.5788 &	 $1^{+}  3^{+} 4^{+} 5^{+}$ &	 55850.4588 &	 1249.6235 &	 $1^{-} 2^{-}$ &	 55869.0104 &	 1408.3737 &	 $1^{-} 2^{-}$ &	 56037.4884 &	 1287.2936 &	 $2^{-}$ \\	
&  &  & 50 &	 56690.8982 &	 859.2445 &	 $2^{-} 3^{+} 4^{+} 5^{+}$ &	\textbf{58120.8249} &	 314.3063 &	 $1^{+}  3^{+} 4^{+} 5^{+}$ &	 55563.0878 &	 1044.9051 &	 $1^{-} 2^{-}$ &	 54981.1206 &	 1223.7431 &	 $1^{-} 2^{-}$ &	 55667.0434 &	 1278.2837 &	 $1^{-} 2^{-}$ \\	
&  & \multirow{2}{*}{0.01}  & 25 &	 56197.0249 &	 738.3354 &	 $2^{-} 3^{+} 4^{+}$ &	\textbf{57528.4355} &	 304.4306 &	 $1^{+}  3^{+} 4^{+} 5^{+}$ &	 54995.4516 &	 1230.9769 &	 $1^{-} 2^{-}$ &	 55013.2733 &	 1387.3025 &	 $1^{-} 2^{-}$ &	 55179.3063 &	 1266.4267 &	 $2^{-}$ \\	
&  &  & 50 &	 54931.1821 &	 829.5866 &	 $2^{-} 3^{+} 4^{+} 5^{+}$ &	\textbf{56312.8274} &	 298.8610 &	 $1^{+}  3^{+} 4^{+} 5^{+}$ &	 53851.9927 &	 1009.2563 &	 $1^{-} 2^{-}$ &	 53287.7928 &	 1187.2402 &	 $1^{-} 2^{-}$ &	 53950.6793 &	 1236.5882 &	 $1^{-} 2^{-}$ \\	
&  & \multirow{2}{*}{0.001}  &  25 &	 53456.0312 &	 705.5039 &	 $2^{-} 3^{+} 4^{+}$ &	\textbf{54715.2628} &	 282.5314 &	 $1^{+}  3^{+} 4^{+} 5^{+}$ &	 52331.0919 &	 1172.9610 &	 $1^{-} 2^{-}$ &	 52346.6392 &	 1321.7096 &	 $1^{-} 2^{-}$ &	 52505.0533 &	 1201.5375 &	 $2^{-}$ \\	
&  &  & 50 &	 49451.2762 &	 737.5644 &	 $2^{-} 3^{+} 4^{+}$ &	\textbf{50683.8705} &	 252.3483 &	 $1^{+}  3^{+} 4^{+} 5^{+}$ &	 48521.4630 &	 900.6466 &	 $1^{-} 2^{-}$ &	 48012.1505 &	 1072.2455 &	 $1^{-} 2^{-}$ &	 48602.1733 &	 1107.2660 &	 $2^{-}$ \\	\hline

\multirow{6}{*}{\rotatebox[origin=c]{90}{strong-500}}  & \multirow{6}{*}{\rotatebox[origin=c]{90}{22223}} & \multirow{2}{*}{0.1} & 25 &	 38822.1695 &	 692.1198 &	 $2^{-} 3^{-} 4^{-} 5^{-}$ &	\textbf{40391.0362} &	 449.8195 &	 $1^{+}  3^{+} 4^{+} 5^{+}$ &	 39792.1607 &	 415.5621 &	 $1^{+} 2^{-}$ &	 39754.0904 &	 424.9780 &	 $1^{+} 2^{-}$ &	 39769.7011 &	 392.0549 &	 $1^{+} 2^{-}$ \\	
&  &  & 50 &	 38444.0651 &	 620.4975 &	 $2^{-} 3^{-} 4^{-} 5^{-}$ &	\textbf{40078.0983} &	 348.7030 &	 $1^{+}  3^{+} 4^{+} 5^{+}$ &	 39442.9758 &	 605.5909 &	 $1^{+} 2^{-}$ &	 39485.7055 &	 483.2892 &	 $1^{+} 2^{-}$ &	 39416.6356 &	 382.3801 &	 $1^{+} 2^{-}$ \\	
&  & \multirow{2}{*}{0.01}  & 25 &	 38026.4154 &	 657.3621 &	 $2^{-} 3^{-} 4^{-} 5^{-}$ &	\textbf{39525.2027} &	 425.9579 &	 $1^{+}  3^{+} 4^{+} 5^{+}$ &	 38973.8435 &	 390.8087 &	 $1^{+} 2^{-}$ &	 38936.9771 &	 399.8035 &	 $1^{+} 2^{-}$ &	 38951.8432 &	 369.6260 &	 $1^{+} 2^{-}$ \\	
&  &  & 50 &	 36864.1232 &	 555.3952 &	 $2^{-} 3^{-} 4^{-} 5^{-}$ &	\textbf{38332.2087} &	 312.3534 &	 $1^{+}  3^{+} 4^{+} 5^{+}$ &	 37800.5870 &	 535.4997 &	 $1^{+} 2^{-}$ &	 37839.6026 &	 424.3667 &	 $1^{+} 2^{-}$ &	 37781.3920 &	 337.2743 &	 $1^{+} 2^{-}$ \\	
&  & \multirow{2}{*}{0.001}  &  25 &	 35546.6995 &	 551.6446 &	 $2^{-} 3^{-} 4^{-} 5^{-}$ &	\textbf{36827.5418} &	 352.7810 &	 $1^{+}  3^{+} 4^{+} 5^{+}$ &	 36424.1740 &	 315.3111 &	 $1^{+} 2^{-}$ &	 36391.1536 &	 322.1788 &	 $1^{+} 2^{-}$ &	 36404.2919 &	 299.3640 &	 $1^{+} 2^{-}$ \\	
&  &  & 50 &	 31947.2385 &	 360.7027 &	 $2^{-} 3^{-} 4^{-} 5^{-}$ &	\textbf{32899.2694} &	 206.7434 &	 $1^{+}$ &	 32685.3625 &	 321.9032 &	 $1^{+}$ &	 32727.4059 &	 242.3192 &	 $1^{+}$ &	 32688.8779 &	 201.6241 &	 $1^{+}$ \\	\bottomrule
\end{tabular}
\end{table*}

%% file: tables/tab2_cher_fixed_delta.tex
\begin{table*}[!t]
    \centering
    \tiny
    \setlength{\tabcolsep}{2pt}
    \caption{Results for same dispersion using Hoeffding bound}
    \label{tab:cher-same-delta}
    \begin{tabular}{l|r|r|l|r r c|r r c|r r c|r r c|r r c}
    \toprule
       ~ & \multicolumn{1}{c|}{\multirow{2}{*}{B}}  &\multicolumn{1}{c|}{\multirow{2}{*}{$\alpha$}} & \multicolumn{1}{c|}{\multirow{2}{*}{$\delta$}} &
        \multicolumn{3}{c|}{GSEMO with g(x) (1)} & \multicolumn{3}{c|}{GSEMO+Filtering with g(x) (2)}  & 
        \multicolumn{3}{c|}{NSGA-II with g(x) (3)}& 
        \multicolumn{3}{c|}{NSGA-II with g$^\prime$(x) (4)}& 
        \multicolumn{3}{c}{NSGA-II with g$^{\prime\prime}$(x) (5)}\\
        ~ & ~ & ~ & ~ & \multicolumn{1}{c}{mean} & \multicolumn{1}{c}{std} & \multicolumn{1}{c|}{stat} & 
        \multicolumn{1}{c}{mean} & \multicolumn{1}{c}{std} & \multicolumn{1}{c|}{stat} & 
        \multicolumn{1}{c}{mean} & \multicolumn{1}{c}{std} & \multicolumn{1}{c|}{stat} & 
        \multicolumn{1}{c}{mean} & \multicolumn{1}{c}{std} & \multicolumn{1}{c|}{stat} & 
        \multicolumn{1}{c}{mean} & \multicolumn{1}{c}{std} & \multicolumn{1}{c}{stat} \\ \hline
\multirow{6}{*}{\rotatebox[origin=c]{90}{uncorr-100}}  & \multirow{6}{*}{\rotatebox[origin=c]{90}{2407}} & \multirow{2}{*}{0.1} & 25 &	 10987.3004 &	 75.7683 &	 $2^{-}$ &	\textbf{11042.7989} &	 0.0000 &	 $1^{+}  3^{+} 4^{+} 5^{+}$ &	 10965.6291 &	 88.0780 &	 $2^{-}$ &	 10956.9290 &	 50.1020 &	 $2^{-}$ &	 10965.9887 &	 51.9898 &	 $2^{-}$ \\	
&  &  & 50 &	 10778.0078 &	 60.1974 &	 $2^{-}$ &	\textbf{10821.5978} &	 0.0000 &	 $1^{+}  3^{+} 4^{+} 5^{+}$ &	 10753.4968 &	 51.9386 &	 $2^{-}$ &	 10728.1263 &	 79.0134 &	 $2^{-}$ &	 10749.4108 &	 48.9647 &	 $2^{-}$ \\	
&  & \multirow{2}{*}{0.01}  & 25 &	 10896.5878 &	 74.5928 &	 $2^{-}$ &	\textbf{10951.1744} &	 0.0000 &	 $1^{+}  3^{+} 4^{+} 5^{+}$ &	 10875.7430 &	 86.4446 &	 $2^{-}$ &	 10867.4019 &	 48.9713 &	 $2^{-}$ &	 10876.2792 &	 50.7360 &	 $2^{-}$ \\	
&  &  & 50 &	 10596.7649 &	 57.6057 &	 $2^{-}$ &	\textbf{10638.3488} &	 0.0000 &	 $1^{+}  3^{+} 4^{+} 5^{+}$ &	 10573.7130 &	 49.3723 &	 $2^{-}$ &	 10549.2659 &	 76.1316 &	 $2^{-}$ &	 10569.9917 &	 46.4640 &	 $2^{-}$ \\	
&  & \multirow{2}{*}{0.001}  &  25 &	 10826.9815 &	 73.6938 &	 $2^{-}$ &	\textbf{10880.8684} &	 0.0000 &	 $1^{+}  3^{+} 4^{+} 5^{+}$ &	 10806.7709 &	 85.1930 &	 $2^{-}$ &	 10798.7053 &	 48.1051 &	 $2^{-}$ &	 10807.4426 &	 49.7746 &	 $2^{-}$ \\	
&  &  & 50 &	 10457.6924 &	 55.6226 &	 $2^{-}$ &	\textbf{10497.7369} &	 0.0000 &	 $1^{+}  3^{+} 4^{+} 5^{+}$ &	 10435.7601 &	 47.4116 &	 $2^{-}$ &	 10412.0216 &	 73.9290 &	 $2^{-}$ &	 10432.3187 &	 44.5489 &	 $2^{-}$ \\	\hline
\multirow{6}{*}{\rotatebox[origin=c]{90}{strong-100}}  & \multirow{6}{*}{\rotatebox[origin=c]{90}{4187}} & \multirow{2}{*}{0.1} & 25 &	 8463.1467 &	 127.6172 &	 $2^{-}$ &	\textbf{8561.1573} &	 85.4504 &	 $1^{+}  3^{+} 4^{+} 5^{+}$ &	 8481.7278 &	 123.3848 &	 $2^{-}$ &	 8456.8099 &	 146.8670 &	 $2^{-}$ &	 8456.3390 &	 122.3538 &	 $2^{-}$ \\	
&  &  & 50 &	 8278.6327 &	 90.2553 &	 $2^{-}$ &	\textbf{8332.4692} &	 64.4783 &	 $1^{+}  3^{+}  5^{+}$ &	 8240.4850 &	 110.0787 &	 $2^{-}$ &	 8276.2257 &	 119.9424 &	 ~  &	 8233.2780 &	 111.0076 &	 $2^{-}$ \\	
&  & \multirow{2}{*}{0.01}  & 25 &	 8369.3409 &	 122.2698 &	 $2^{-}$ &	\textbf{8464.4621} &	 81.5096 &	 $1^{+}  3^{+} 4^{+} 5^{+}$ &	 8387.9199 &	 118.6225 &	 $2^{-}$ &	 8364.1958 &	 141.2465 &	 $2^{-}$ &	 8363.9441 &	 117.6448 &	 $2^{-}$ \\	
&  &  & 50 &	 8086.8101 &	 82.3211 &	 $2^{-}$ &	\textbf{8139.0049} &	 58.6173 &	 $1^{+}  3^{+}  5^{+}$ &	 8054.9801 &	 100.8829 &	 $2^{-}$ &	 8087.7692 &	 109.6779 &	 ~  &	 8048.4883 &	 101.6750 &	 $2^{-}$ \\	
&  & \multirow{2}{*}{0.001}  &  25 &	 8297.3868 &	 118.1961 &	 $2^{-}$ &	\textbf{8390.2653} &	 78.4972 &	 $1^{+}  3^{+} 4^{+} 5^{+}$ &	 8315.9386 &	 114.9756 &	 $2^{-}$ &	 8293.1306 &	 136.9396 &	 $2^{-}$ &	 8293.0470 &	 114.0400 &	 $2^{-}$ \\	
&  &  & 50 &	 7939.6195 &	 76.3766 &	 $2^{-}$ &	\textbf{7990.5545} &	 54.1630 &	 $1^{+}  3^{+}  5^{+}$ &	 7912.6372 &	 93.8879 &	 $2^{-}$ &	 7943.1615 &	 101.8268 &	 ~  &	 7906.6941 &	 94.5672 &	 $2^{-}$ \\	\hline
\multirow{6}{*}{\rotatebox[origin=c]{90}{uncorr-300}}  & \multirow{6}{*}{\rotatebox[origin=c]{90}{6853}} & \multirow{2}{*}{0.1} & 25 &	 33863.1544 &	 205.0835 &	 $2^{-}$ &	\textbf{34212.8177} &	 146.9244 &	 $1^{+}  3^{+} 4^{+} 5^{+}$ &	 33610.2955 &	 532.9802 &	 $2^{-}$ &	 33600.5469 &	 554.0346 &	 $2^{-}$ &	 33545.0211 &	 487.7131 &	 $2^{-}$ \\	
&  &  & 50 &	 33428.2570 &	 259.2905 &	 $2^{-}$ &	\textbf{33821.1723} &	 157.7349 &	 $1^{+}  3^{+} 4^{+} 5^{+}$ &	 33276.1084 &	 510.0035 &	 $2^{-}$ &	 33143.0191 &	 448.4365 &	 $2^{-}$ &	 33177.8089 &	 480.9303 &	 $2^{-}$ \\	
&  & \multirow{2}{*}{0.01}  & 25 &	 33708.5096 &	 203.9304 &	 $2^{-}$ &	\textbf{34055.7967} &	 145.6324 &	 $1^{+}  3^{+} 4^{+} 5^{+}$ &	 33457.5385 &	 530.1364 &	 $2^{-}$ &	 33448.1219 &	 551.0166 &	 $2^{-}$ &	 33392.9168 &	 484.9026 &	 $2^{-}$ \\	
&  &  & 50 &	 33119.8295 &	 255.9406 &	 $2^{-}$ &	\textbf{33507.4493} &	 154.5170 &	 $1^{+}  3^{+} 4^{+} 5^{+}$ &	 32970.6017 &	 504.5017 &	 $2^{-}$ &	 32839.2289 &	 443.7443 &	 $2^{-}$ &	 32873.6003 &	 475.3176 &	 $2^{-}$ \\	
&  & \multirow{2}{*}{0.001}  &  25 &	 33589.8465 &	 203.0502 &	 $2^{-}$ &	\textbf{33935.3102} &	 144.6467 &	 $1^{+}  3^{+} 4^{+} 5^{+}$ &	 33340.3278 &	 527.9571 &	 $2^{-}$ &	 33331.1621 &	 548.7016 &	 $2^{-}$ &	 33276.2031 &	 482.7469 &	 $2^{-}$ \\	
&  &  & 50 &	 32883.1649 &	 253.3855 &	 $2^{-}$ &	\textbf{33266.7212} &	 152.0660 &	 $1^{+}  3^{+} 4^{+} 5^{+}$ &	 32736.1783 &	 500.2836 &	 $2^{-}$ &	 32606.1226 &	 440.1485 &	 $2^{-}$ &	 32640.1729 &	 471.0147 &	 $2^{-}$ \\	\hline
\multirow{6}{*}{\rotatebox[origin=c]{90}{strong-300}}  & \multirow{6}{*}{\rotatebox[origin=c]{90}{13821}} & \multirow{2}{*}{0.1} & 25 &	 23744.0892 &	 430.4747 &	 $2^{-} 3^{-}  5^{-}$ &	\textbf{24300.5479} &	 214.9589 &	 $1^{+}  3^{+} 4^{+}$ &	 24033.7421 &	 324.1660 &	 $1^{+} 2^{-}$ &	 23921.4833 &	 341.5510 &	 $2^{-}$ &	 24109.9465 &	 230.6676 &	 $1^{+}$ \\	
&  &  & 50 &	 23462.9510 &	 329.9513 &	 $2^{-} 3^{-}  5^{-}$ &	\textbf{23997.1125} &	 215.9394 &	 $1^{+}  3^{+} 4^{+} 5^{+}$ &	 23735.7246 &	 288.1238 &	 $1^{+} 2^{-}$ &	 23566.2352 &	 299.2535 &	 $2^{-}$ &	 23767.0263 &	 219.6354 &	 $1^{+} 2^{-}$ \\	
&  & \multirow{2}{*}{0.01}  & 25 &	 23603.7492 &	 424.5461 &	 $2^{-} 3^{-}  5^{-}$ &	\textbf{24152.7138} &	 210.6726 &	 $1^{+}  3^{+} 4^{+}$ &	 23892.9269 &	 317.9729 &	 $1^{+} 2^{-}$ &	 23782.5354 &	 335.8005 &	 $2^{-}$ &	 23967.9043 &	 226.0936 &	 $1^{+}$ \\	
&  &  & 50 &	 23181.2241 &	 317.9412 &	 $2^{-} 3^{-}  5^{-}$ &	\textbf{23700.6927} &	 206.3054 &	 $1^{+}  3^{+} 4^{+} 5^{+}$ &	 23452.1877 &	 276.3622 &	 $1^{+} 2^{-}$ &	 23289.4860 &	 288.4378 &	 $2^{-}$ &	 23482.9419 &	 210.5250 &	 $1^{+} 2^{-}$ \\	
&  & \multirow{2}{*}{0.001}  &  25 &	 23496.0973 &	 420.0184 &	 $2^{-} 3^{-}  5^{-}$ &	\textbf{24039.3181} &	 207.2892 &	 $1^{+}  3^{+} 4^{+}$ &	 23784.9190 &	 313.2317 &	 $1^{+} 2^{-}$ &	 23675.9269 &	 331.3992 &	 $2^{-}$ &	 23858.9115 &	 222.5879 &	 $1^{+}$ \\	
&  &  & 50 &	 22965.0474 &	 308.7926 &	 $2^{-} 3^{-}  5^{-}$ &	\textbf{23473.2418} &	 198.9571 &	 $1^{+}  3^{+} 4^{+} 5^{+}$ &	 23234.6394 &	 267.3758 &	 $1^{+} 2^{-}$ &	 23077.1290 &	 280.1499 &	 $2^{-}$ &	 23264.9563 &	 203.5545 &	 $1^{+} 2^{-}$ \\	\hline
\multirow{6}{*}{\rotatebox[origin=c]{90}{uncorr-500}}  & \multirow{6}{*}{\rotatebox[origin=c]{90}{11243}} & \multirow{2}{*}{0.1} & 25 &	 56985.6975 &	 747.8288 &	 $2^{-} 3^{+} 4^{+}$ &	\textbf{58337.8820} &	 310.8352 &	 $1^{+}  3^{+} 4^{+} 5^{+}$ &	 55761.8933 &	 1247.6913 &	 $1^{-} 2^{-}$ &	 55780.3693 &	 1406.1906 &	 $1^{-} 2^{-}$ &	 55948.9616 &	 1285.1401 &	 $2^{-}$ \\	
&  &  & 50 &	 56509.1689 &	 856.1817 &	 $2^{-} 3^{+} 4^{+} 5^{+}$ &	\textbf{57934.0703} &	 312.7435 &	 $1^{+}  3^{+} 4^{+} 5^{+}$ &	 55386.3718 &	 1041.2515 &	 $1^{-} 2^{-}$ &	 54806.2623 &	 1219.9708 &	 $1^{-} 2^{-}$ &	 55489.9899 &	 1273.9787 &	 $1^{-} 2^{-}$ \\	
&  & \multirow{2}{*}{0.01}  & 25 &	 56790.6294 &	 745.4714 &	 $2^{-} 3^{+} 4^{+}$ &	\textbf{58137.6851} &	 309.2459 &	 $1^{+}  3^{+} 4^{+} 5^{+}$ &	 55572.3326 &	 1243.5564 &	 $1^{-} 2^{-}$ &	 55590.6468 &	 1401.5183 &	 $1^{-} 2^{-}$ &	 55758.8493 &	 1280.5162 &	 $2^{-}$ \\	
&  &  & 50 &	 56119.2292 &	 849.6132 &	 $2^{-} 3^{+} 4^{+} 5^{+}$ &	\textbf{57533.3999} &	 309.3150 &	 $1^{+}  3^{+} 4^{+} 5^{+}$ &	 55007.2442 &	 1033.3396 &	 $1^{-} 2^{-}$ &	 54431.0658 &	 1211.8786 &	 $1^{-} 2^{-}$ &	 55109.7653 &	 1264.7366 &	 $1^{-} 2^{-}$ \\	
&  & \multirow{2}{*}{0.001}  &  25 &	 56640.9485 &	 743.6632 &	 $2^{-} 3^{+} 4^{+}$ &	\textbf{57984.0686} &	 308.0287 &	 $1^{+}  3^{+} 4^{+} 5^{+}$ &	 55426.8776 &	 1240.3841 &	 $1^{-} 2^{-}$ &	 55445.0676 &	 1397.9335 &	 $1^{-} 2^{-}$ &	 55612.9710 &	 1276.9688 &	 $2^{-}$ \\	
&  &  & 50 &	 55820.0179 &	 844.5763 &	 $2^{-} 3^{+} 4^{+} 5^{+}$ &	\textbf{57226.0199} &	 306.6676 &	 $1^{+}  3^{+} 4^{+} 5^{+}$ &	 54716.3294 &	 1027.2713 &	 $1^{-} 2^{-}$ &	 54143.1674 &	 1205.6714 &	 $1^{-} 2^{-}$ &	 54818.0086 &	 1257.6477 &	 $1^{-} 2^{-}$ \\	\hline
\multirow{6}{*}{\rotatebox[origin=c]{90}{strong-500}}  & \multirow{6}{*}{\rotatebox[origin=c]{90}{22223}} & \multirow{2}{*}{0.1} & 25 &	 38739.7417 &	 688.5044 &	 $2^{-} 3^{-} 4^{-} 5^{-}$ &	\textbf{40301.3374} &	 447.3446 &	 $1^{+}  3^{+} 4^{+} 5^{+}$ &	 39707.3872 &	 412.9829 &	 $1^{+} 2^{-}$ &	 39669.4439 &	 422.3664 &	 $1^{+} 2^{-}$ &	 39685.3277 &	 389.7418 &	 $1^{+} 2^{-}$ \\	
&  &  & 50 &	 38280.9153 &	 613.7020 &	 $2^{-} 3^{-} 4^{-} 5^{-}$ &	\textbf{39897.8123} &	 344.8981 &	 $1^{+}  3^{+} 4^{+} 5^{+}$ &	 39273.3775 &	 598.3347 &	 $1^{+} 2^{-}$ &	 39315.7031 &	 477.1865 &	 $1^{+} 2^{-}$ &	 39247.8886 &	 377.7601 &	 $1^{+} 2^{-}$ \\	
&  & \multirow{2}{*}{0.01}  & 25 &	 38563.3178 &	 680.7776 &	 $2^{-} 3^{-} 4^{-} 5^{-}$ &	\textbf{40109.3511} &	 442.0526 &	 $1^{+}  3^{+} 4^{+} 5^{+}$ &	 39525.9425 &	 407.4683 &	 $1^{+} 2^{-}$ &	 39488.2712 &	 416.7806 &	 $1^{+} 2^{-}$ &	 39504.1345 &	 384.7767 &	 $1^{+} 2^{-}$ \\	
&  &  & 50 &	 37930.8422 &	 599.1733 &	 $2^{-} 3^{-} 4^{-} 5^{-}$ &	\textbf{39510.9695} &	 336.7709 &	 $1^{+}  3^{+} 4^{+} 5^{+}$ &	 38909.4678 &	 582.7780 &	 $1^{+} 2^{-}$ &	 38950.9755 &	 464.0994 &	 $1^{+} 2^{-}$ &	 38885.5815 &	 367.8049 &	 $1^{+} 2^{-}$ \\	
&  & \multirow{2}{*}{0.001}  &  25 &	 38427.9430 &	 674.8592 &	 $2^{-} 3^{-} 4^{-} 5^{-}$ &	\textbf{39962.0501} &	 437.9920 &	 $1^{+}  3^{+} 4^{+} 5^{+}$ &	 39386.7151 &	 403.2421 &	 $1^{+} 2^{-}$ &	 39349.2525 &	 412.4983 &	 $1^{+} 2^{-}$ &	 39365.1001 &	 380.9688 &	 $1^{+} 2^{-}$ \\	
&  &  & 50 &	 37662.2216 &	 588.0761 &	 $2^{-} 3^{-} 4^{-} 5^{-}$ &	\textbf{39214.1346} &	 330.5707 &	 $1^{+}  3^{+} 4^{+} 5^{+}$ &	 38630.2301 &	 570.8539 &	 $1^{+} 2^{-}$ &	 38671.1102 &	 454.0722 &	 $1^{+} 2^{-}$ &	 38607.6080 &	 360.1434 &	 $1^{+} 2^{-}$ \\	\bottomrule

\end{tabular}
\end{table*}

%% file: tables/tab5_rand_del.tex
\begin{table*}[!t]
    \centering
    \scriptsize
    \caption{Results for different dispersion using Chebyshev's inequality}
    
    \label{tab:rand_del}
    \setlength{\tabcolsep}{2pt}
    \begin{tabular}{c|r|r|r r c|r r c|r r c|r r c}
    \toprule
    \multicolumn{1}{c|}{\multirow{2}{*}{Instance}} & \multicolumn{1}{c|}{\multirow{2}{*}{B}} & \multicolumn{1}{c|}{\multirow{2}{*}{$\alpha$}} &
        \multicolumn{3}{c|}{GSEMO with g(x) (1)}  & 
        \multicolumn{3}{c|}{GSEMO+Filtering with g(x) (2)}  & 
        \multicolumn{3}{c|}{NSGA-II with g(x) (3)} &
        \multicolumn{3}{c}{NSGA-II with g$^{\prime}$(x) (4)} \\
        
        ~ & ~ & ~ & \multicolumn{1}{c}{mean} & \multicolumn{1}{c}{std} & \multicolumn{1}{c|}{stat} & 
        \multicolumn{1}{c}{mean} & \multicolumn{1}{c}{std} & \multicolumn{1}{c|}{stat} & 
        \multicolumn{1}{c}{mean} & \multicolumn{1}{c}{std} & \multicolumn{1}{c|}{stat} &
        \multicolumn{1}{c}{mean} & \multicolumn{1}{c}{std} & \multicolumn{1}{c}{stat} \\ \hline

\multirow{3}{*}{uncorr-100} & \multirow{3}{*}{2407} & 0.1 &	 8546.7015 &	 322.7956 &	 ~  &	 8562.9620 &	 323.1475 &	 ~  &	\textbf{8565.8942} &	 324.9223 &	 ~  &	 8564.2549 &	 323.8494 &	 ~  \\	
& & 0.01 &	\textbf{4682.2710} &	 799.1875 &	 ~  &	 4675.9295 &	 795.5029 &	 ~  &	 4674.8735 &	 797.5334 &	 ~  &	 4677.1350 &	 799.3174 &	 ~  \\	
& &  0.001 &	\textbf{1967.4965} &	 880.0104 &	 ~  &	 1962.3962 &	 886.0726 &	 ~  &	 1919.2445 &	 917.6798 &	 ~  &	 1953.6975 &	 894.4411 &	 ~  \\	\hline
\multirow{3}{*}{strong-100} & \multirow{3}{*}{4187} & 0.1 &	 7100.1631 &	 245.6990 &	 ~  &	 7099.8604 &	 250.2445 &	 ~  &	 7122.7698 &	 249.2182 &	 ~  &	\textbf{7123.3386} &	 250.5080 &	 ~  \\	
& & 0.01 &	 5315.4258 &	 326.6578 &	 ~  &	 5308.5248 &	 343.0757 &	 ~  &	 5329.4856 &	 326.9290 &	 ~  &	\textbf{5333.3263} &	 325.5847 &	 ~  \\	
& & 0.001 &	\textbf{3746.9429} &	 732.1072 &	 ~  &	 3746.1213 &	 731.9397 &	 ~  &	 3734.2093 &	 731.9080 &	 ~  &	 3743.5859 &	 731.6689 &	 ~  \\	\hline
\multirow{3}{*}{uncorr-300} & \multirow{3}{*}{6853} & 0.1 &	 29089.9230 &	 479.0749 &	 $ 2^{-} 3^{-} 4^{-}$ &	 29580.8035 &	 355.6388 &	 $1^{+}   $ &	\textbf{29735.0392} &	 358.7213 &	 $1^{+}   $ &	 29676.9763 &	 388.6763 &	 $1^{+}   $ \\	
& & 0.01 &	 19396.8836 &	 1128.3967 &	 ~  &	 19431.8036 &	 1126.0478 &	 ~  &	\textbf{19589.0418} &	 1131.0819 &	 ~  &	 19580.3092 &	 1131.0342 &	 ~  \\	
& & 0.001 &	 8910.0087 &	 1370.4531 &	 ~  &	 8850.8254 &	 1381.8863 &	 ~  &	 8899.5025 &	 1378.4832 &	 ~  &	\textbf{8971.3050} &	 1383.6696 &	 ~  \\	\hline
\multirow{3}{*}{strong-300} & \multirow{3}{*}{13821}& 0.1 &	 21789.9295 &	 335.7672 &	 $ 2^{-} 3^{-} 4^{-}$ &	 22171.9138 &	 358.9991 &	 $1^{+}   $ &	\textbf{22345.3397} &	 309.3574 &	 $1^{+}   $ &	 22297.4300 &	 307.0689 &	 $1^{+}   $ \\	
& & 0.01 &	 18172.8378 &	 560.3248 &	 ~  &	 18195.2974 &	 615.1645 &	 ~  &	 18338.0359 &	 588.6787 &	 ~  &	\textbf{18342.4421} &	 576.7652 &	 ~  \\	
& & 0.001 &	 14629.7944 &	 809.3377 &	 ~  &	 14617.5558 &	 794.6289 &	 ~  &	 14643.0814 &	 808.6424 &	 ~  &	\textbf{14667.4349} &	 812.9751 &	 ~  \\	\hline
\multirow{3}{*}{uncorr-500} & \multirow{3}{*}{11243} & 0.1 &	 50266.4398 &	 709.0211 &	 $ 2^{-} 3^{-} 4^{-}$ &	\textbf{52494.0984} &	 556.8082 &	 $1^{+}   $ &	 52468.0194 &	 532.9634 &	 $1^{+}   $ &	 52149.3408 &	 700.4027 &	 $1^{+}   $ \\	
& & 0.01 &	 37753.4240 &	 1566.1944 &	 ~  &	 38510.4882 &	 1564.4777 &	 ~  &	\textbf{38746.1887} &	 1539.8167 &	 ~  &	 38686.7230 &	 1555.8618 &	 ~  \\	
& & 0.001 &	 18969.0800 &	 2144.1783 &	 ~  &	 18880.7433 &	 2144.3506 &	 ~  &	 19153.0886 &	 2137.5579 &	 ~  &	\textbf{19190.4696} &	 2134.7447 &	 ~  \\	\hline
\multirow{3}{*}{strong-500} & \multirow{3}{*}{22223} & 0.1 &	 35919.5415 &	 631.8822 &	 $ 2^{-} 3^{-} 4^{-}$ &	\textbf{37833.8138} &	 352.1352 &	 $1^{+}   $ &	 37832.1320 &	 332.8651 &	 $1^{+}   $ &	 37690.0363 &	 317.8082 &	 $1^{+}   $ \\	
& & 0.01 &	 30977.9111 &	 679.1163 &	 $ 2^{-} 3^{-} 4^{-}$ &	 31554.3119 &	 682.8664 &	 $1^{+}   $ &	\textbf{31822.0362} &	 649.7576 &	 $1^{+}   $ &	 31805.6899 &	 637.4521 &	 $1^{+}   $ \\	
& & 0.001 &	 25041.2112 &	 721.5121 &	 ~  &	 25018.0126 &	 704.1720 &	 ~  &	 25131.2311 &	 723.9024 &	 ~  &	\textbf{25193.7178} &	 741.8816 &	 ~  \\	\bottomrule 

\end{tabular}
\end{table*}